\documentclass[letterpaper]{article} % DO NOT CHANGE THIS
\usepackage{aaai2027}
\usepackage[hyphens]{url}  % DO NOT CHANGE THIS
\usepackage{graphicx} % DO NOT CHANGE THIS
\usepackage{natbib}  % DO NOT CHANGE THIS AND DO NOT ADD ANY OPTIONS TO IT
\usepackage{caption} % DO NOT CHANGE THIS AND DO NOT ADD ANY OPTIONS TO IT
\usepackage{algorithm}
\usepackage{algorithmic}
\usepackage[version=4]{mhchem}
\usepackage{newfloat}
\usepackage{amsfonts}
\usepackage{listings}
\DeclareCaptionStyle{ruled}{labelfont=normalfont,labelsep=colon,strut=off} % DO NOT CHANGE THIS
\floatstyle{ruled}
\newfloat{listing}{tb}{lst}{}
\floatname{listing}{Listing}

\usepackage{booktabs}
\usepackage{multirow}
\usepackage{enumitem}
\usepackage{booktabs, makecell}
\usepackage{amsthm} 
\newtheorem{theorem}{Theorem}
\newtheorem{proposition}[theorem]{Proposition}
\newtheorem{corollary}[theorem]{Corollary}
\usepackage{amsmath}
\usepackage[table]{xcolor} % 必须引入
\usepackage{pifont}
\usepackage{graphicx}
\usepackage{array}
\newcommand{\mstd}[2]{#1\scriptsize{\textcolor{gray}{$\pm$#2}}\normalsize}
\title{ NICE: Scale-Stable Perturbations for Graph Neural Network Explanations via \\ Noise Corruption }

\author{%
  \textbf{Ziluowen Luo$^1$}, \textbf{Jun Yin$^2$}, \textbf{Ruochen Liu$^1$}, \textbf{Ming Cheng$^1$},\\
  \textbf{Shirui Pan$^3$}, \textbf{Chengqi Zhang$^2$}, \textbf{Senzhang Wang$^{1}\thanks{Corresponding author}$}\\
}
\affiliations{
    \textsuperscript{\rm 1}Central South University, \textsuperscript{\rm 2}Hong Kong Polytechnic University, \textsuperscript{\rm 3}Griffith University \\
     \texttt{\{lzlwddl, szwang, ruochen, 244701028\}@csu.edu.cn} , \texttt{Junmay.yin@connect.polyu.hk} \\
     \texttt{chengqi.zhang@polyu.edu.hk} , \texttt{s.pan@griffith.edu.au}
}

\nocopyright
\begin{document}

\maketitle

\begin{abstract}
Post-hoc Graph Neural Network (GNN) explainers commonly follow a \textit{Perturb-Query} paradigm, inferring the importance of graph elements based on queried predictions to perturbed inputs. 
However, such perturbations often introduce substantial distribution shift, undermining the reliability of the queried predictions used to derive explanations. 
While existing efforts mainly improve perturbed graphs or stabilize model predictions on them, we revisit the perturbation mechanism itself. 
We show that the widely used \textbf{Element-wise Masking} (EM) suppresses edge-induced messages toward zero, causing deterministic scale contraction that accumulates across message-passing layers, a phenomenon we term \textit{Scale Drift}. 
Consequently, prediction changes under EM may conflate information corruption with deviations in propagation scale. 
As a scale-stable alternative to EM, we introduce \textbf{Noise Corruption} (NC), which perturbs each message through matched-norm random-direction corruption while preserving the expected squared message norm. 
Building on NC, we propose \textbf{NICE}, a Noise Corruption-based explanation framework, which learns a Stochastic Restoration Boundary (SRB) under NC-induced uncertainty, balancing target-prediction restoration against compactness. 
Furthermore, Boundary-Integrated Gradient (BIG) converts this boundary into edge attributions by accumulating each edge’s contribution to reducing restoration risk along the restoration path. 
Experiments across multiple benchmarks demonstrate stronger explanation performance and model faithfulness while confirming that NC substantially reduces the Scale Drift induced by masking.
\end{abstract}

\section{Introduction}
Graph Neural Networks (GNNs) have achieved remarkable success across a wide range of graph learning tasks \cite{kipf2016semi,shi2022h2}. 
However, the decision logic of GNNs remains difficult to understand due to the highly entangled message-passing process over graph structures. 
Post-hoc GNN explainers therefore aim to identify the graph elements or structures most responsible for a target prediction \cite{gnne, pge, ReFine, subgraphX, wu2021counterfactual}. 
A broad family of post-hoc explainers follow a Perturb-Query paradigm. 
As illustrated by Figure \ref{fig:intro}, an explainer \textit{(i)} generates element-wise scores, \textit{(ii)} uses them to perturb edge-induced messages during propagation and \textit{(iii)} updates them using the queried predictions \cite{cf^2-counter,10.1109/TPAMI.2022.3204236}. 
The reliability of the resulting explanations thus relys on whether the prediction changes faithfully reflect the message  affected by perturbation.

This assumption is fragile because perturbed graphs can introduce substantial distribution shift, moving the queried predictions away from the model’s original operating regime and making the resulting prediction changes ambiguous \cite{robust_f, OAR, co-operative}. 
To alleviate this issue, existing methods mainly focus on improving explanation quality through better generation of explanation scores and explainer optimization. 
For example, they aim to make explanation subgraphs structurally more natural \cite{d4e}, closer to the original graph distribution \cite{proxy}, or equipped with confidence-aware correction mechanisms \cite{Zhang2025IsYE}. 
However, a crucial factor remains underexplored: whether the observed distribution shift is already introduced by the perturbation mechanism itself, before the resulting signals are used to optimize the explainer.

\begin{figure}[tbp]
    \centering
    \includegraphics[width=1\linewidth]{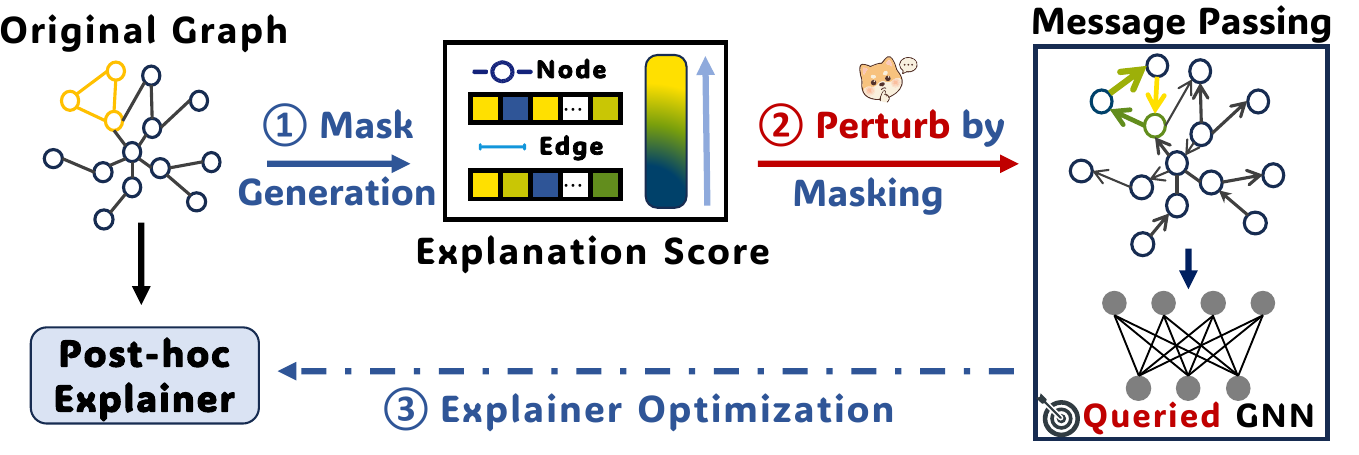}
    \caption{
    Rethinking where the distribution shift enters the \textit{Perturb-Query} pipeline.
Existing efforts focus on \ding{172} or \ding{174} (e.g. ProxyExplainer \cite{proxy}), while this work focuses on \ding{173} the perturbation mechanism itself.}
    \label{fig:intro} 
    \vspace{-0.3cm}
\end{figure}

To isolate whether the perturbation mechanism itself contributes to this shift, we design an experiment using ground-truth explanations as oracle baseline and compare the resulting graph representations with those of the original graph. 
The discrepancy persists under oracle scores, suggesting that score quality alone does not explain the observed shift. 
We identify this effect to Element-wise Masking (EM), the default mechanism used to instantiate explanation scores during message passing. 
As show on the left of Figure \ref{fig:nc_intro}, EM suppresses edge-induced messages toward zero, coupling information corruption with deterministic scale contraction. 
As this contraction accumulates across layers, the perturbed computation progressively departs from the clean message-scale regime, a phenomenon we term \textit{Scale Drift}. 
Hence, these findings motivate us to explore \textit{an alternative perturbation mechanism that corrupts the information carried by messages without explicitly reducing their scale}.

To this end, we introduce \textbf{Noise Corruption} (NC), a message-level perturbation mechanism that replaces zero-directed masking with matched-norm stochastic corruption. 
As show on the right of Figure \ref{fig:nc_intro}, NC combines each clean message with a scale-matched corrupted counterpart through a restoration gate, perturbing the information it carries while preserving its expected squared norm. 
It therefore defines a scale-stable corruption and restoration space that avoids the deterministic scale collapse of EM. 
However, since the corrupted counterparts are sampled stochastically, a fixed restoration configuration can produce different queried predictions across forward passes. 
NC thus provides the perturbation space, but not yet a stable explanation.
% In this work, we introduce \textbf{Noise Corruption} (NC), a message-level perturbation mechanism that replaces deterministic masking with stochastic noise injection. 
% \textcolor{blue}{Instead of asking which edges can be removed by suppressing their messages toward zero, we ask which messages must remain clean under controlled corruption. }
% As illustrated in Figure \ref{fig:nc_intro}, given a message induced by an edge, NC mixes the clean message with controlled noise, thereby perturbing message content without directly collapsing its scale. 
% In this way, NC provides a scale-stable alternative to conventional masking while still allowing the explainer to query how the target model reacts when the information carried by the graph is corrupted.

\begin{figure}
  \centering
  \includegraphics[width=0.45\textwidth]{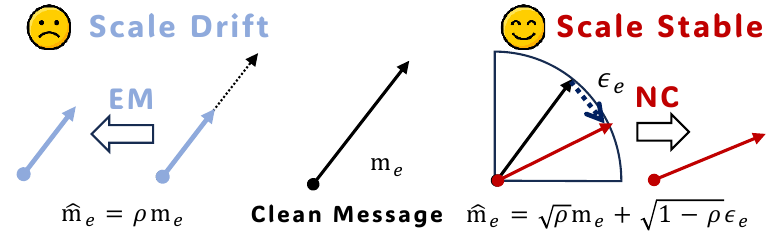} % 图片略小于容器
  \caption{Comparison \textbf{NC} with \textbf{EM} in the message passing.  }
  \label{fig:nc_intro}\vspace{-0.2cm}
\end{figure}

Building on NC, we propose NICE, a Noise Corruption based explanation framework to address two remaining questions separately: how much messages should be jointly restored under stochastic corruption and how their restoration should be attributed to individual edges. 
First, Stochastic Restoration Boundary Learning (SRB) learns a compact boundary by balancing the degradation and variation of the original target prediction against restoration compactness. 
Since the learned gates record restoration degrees rather than their prediction-level contributions, Boundary-Integrated Gradient (BIG) integrates the reduction in stochastic restoration risk along the path from the fully corrupted state to the learned boundary. 
NICE thus first locates a compact stochastic boundary and then converts the restoration process into attributions. 
Our main contributions as follows:
\begin{enumerate}[leftmargin=*,noitemsep]
    % \item We identify \textbf{Scale Drift} as a key source of distribution shift in perturbation-query GNN explainers, showing that it arises from the widely used  masking mechanism.
    \item We identify Scale Drift as a previously overlooked confounding factor in Perturb-Query paradigm, showing that it arises from the widely used  masking mechanism.
    
    % \item We introduce \textit{Noise Corruption} (NC), a scale-stable alternative to conventional masking. By replacing zero-directed message suppression with matched-norm random-direction corruption, NC perturbs the information carried by edge-induced messages while preserving their expected squared norm.
    \item We propose \textit{Noise Corruption} (NC), a scale-stable alternative which replaces zero-directed masking with matched-norm random-direction corruption, thereby defining a scale-stable corruption and restoration space.
    
    \item  Building on NC, we propose NICE, a noise corruption based explanation paradigm that first learns \textit{Stochastic Restoration Boundary} and then generates explanation scores through \textit{Boundary Integrated Gradient}. 
    
    \item Across eight benchmarks, NICE improves average Recall and AUC-ROC over the best baseline by 8.67\% and 7.57\% respectively and obtains higher retention fidelity.
\end{enumerate}

\begin{figure*}[tbp]
        \centering
        \includegraphics[width=1\linewidth]{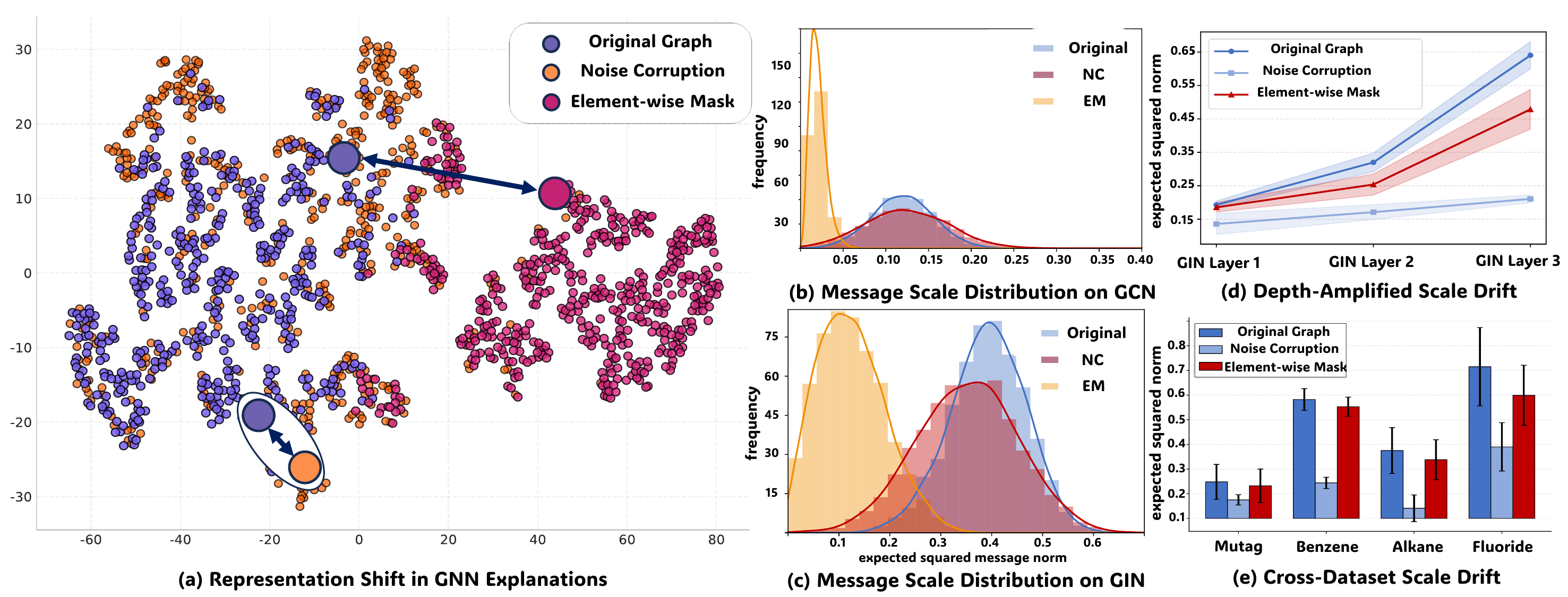}
        \caption{
\textbf{Preliminary experimental results.} 
(a) t-SNE visualization of graph embeddings from the original graph and the same ground-truth explanation instantiated with EM and NC. 
(b)-(c) Message scale distribution shift across GCN and GIN. 
(d) Layer-wise mean scale under different perturbation mechanisms. 
(e) Cross-dataset summary of message-scale statistics. 
}
        \label{fig:motivation}
        \vspace{-0.1cm}
\end{figure*}

\section{Preliminaries}\label{sec:preliminaries}

\paragraph{Notation.} 
Let $G=(\mathcal V,\mathcal E)$ be an input graph where $\mathcal{V} = \{ v_1, v_2, \dots, v_n \}$ denotes the node set and $\mathcal{E} \subseteq \mathcal{V} \times \mathcal{V}$ represents the edge set. 
In this work, we focus on explaining for graph classification tasks that GNN $f(G)\in\mathbb R^C$ was trained to output its prediction $y^*=\arg\max_c f_c(G)$. 
At layer $l$, the GNN computes edge-induced messages and aggregates them to update node representations. 
For an edge $e_{ij} = (v_i,v_j)\in\mathcal E$, we denote the message between them as
\begin{equation}
\mathbf m_{i\leftrightarrow j}^{(l)} = \phi^{(l)} ( \mathbf h_i^{(l)},\mathbf h_j^{(l)}, e_{ij} ),
\end{equation}
where $\mathbf h^{(l)}$ denotes layer-$l$ node representations and $\phi^{(l)}$ is the message function. 
We use the mean squared norm of edge-induced messages as a layer-wise proxy for message \emph{scale}.
\begin{equation}
\mathcal S^{(l)}(G)
=
\mathbb E_{(u,v)\in\mathcal E}
\left[
\left\|
\mathbf m_{u\leftrightarrow v}^{(l)}
\right\|_2^2
\right].
\label{eq:scale}
\end{equation}

\paragraph{Perturb-and-Query Explanation.}
Without lossing generality \cite{abrate2021counterfactual, he2022explainer}, we focus on edge-level explanation that the goal is to assigns each edge $e\in\mathcal E$ an attribution score $s_e\in[0,1]$ to indicate its contribution to the target prediction $y^*$. 
A broad family of explainers follows the \emph{perturb-and-query} paradigm that estimate edge importance by perturbing graph components and querying the response of the fixed target GNN. 
Since GNN predictions are produced through message passing, we formulate perturbation at the message level. 
For simplicity, we regard each edge $e=(v_i,v_j)\in\mathcal E$ as a message-passing edge and write $\mathbf m_e^{(l)}:=\mathbf m_{i\leftrightarrow j}^{(l)}$. 
A generic perturbation mechanism operating at the message-passing level can be written as
\begin{equation}
\widetilde{\mathbf m}_{e}^{(l)}
=
\Psi_{\rho_e}^{(l)}
\left(
\mathbf m_{e}^{(l)}
\right),
\label{eq:peturbation}
\end{equation}
where $\rho_e\in[0,1]$ is an intervention variable controlling how much the message induced by edge $e$ is perturbed and $\Psi_{\rho_e}^{(l)}$ denotes the corresponding perturbation mechanism.

\paragraph{Element-wise Masking.}
In conventional perturb-query explainers, the intervention variable is commonly implemented as a multiplicative mask on the induced message:
\begin{equation}
\Psi_{\mathrm{EM},\rho_e}^{(l)}
\left(
\mathbf m_e^{(l)}
\right)
=
\rho_e \mathbf m_e^{(l)} .
\label{eq:em}
\end{equation}
The learned mask coefficient is directly used as the attribution score, i.e., $s_e=\rho_e$. 
This coupling makes edge attribution operationally tied to multiplicative message suppression. 
In this work, we show that this perturbation mechanism can systematically alter the propagation scale of the target GNN.

\section{Rethinking the Perturbation Mechanism}\label{sec:rethinking_ood}

% To investigate the distribution shift in current perturb-query GNN explanation, we first revisit this problem under oracle explanations and then trace it back to the widely used element-wise masking mechanism that gives rise to Scale Drift.

In this section, we revisit the distribution shift in perturb-query explanation from the perspective of the perturbation mechanism itself. 
We first show that shift persists even when perturbations are instantiated from oracle explanations, suggesting that the issue can't be attributed solely to poor quality. 
We then trace this shift to element-wise masking which deterministically contracts message scale during propagation. 
This phenomenon which we call \emph{Scale Drift}, motivates the need for a scale-stable message perturbation mechanism.

\subsection{Distribution Shift Persists under Oracle Settings} \label{sec:oracle}

Existing efforts mainly attribute unreliable model responses on perturbed graphs to the quality of generated explanations or to the naturalness of perturbed graphs. 
Accordingly, prior efforts mainly improve the learned explanation scores, make perturbed graphs closer to the data distribution or stabilize the queried model responses after perturbation~\cite{dir, d4e}. 
However, these treatments don't isolate whether the observed shift is caused by imperfect explanation or by the perturbation mechanism itself. 

\paragraph{Oracle Experiments.} 
To disentangle these two possibilities, we remove the uncertainty from explanation learning by using ground-truth explanations as oracle explanations. 
Specifically, we instantiate perturbation directly from the ground-truth explanation and compare the resulting graph representations with those of the original graph. 
As shown in Figure~\ref{fig:motivation}(a), a clear representation shift still appears even under this oracle setting. 
This observation suggests that the discrepancy is not merely a byproduct of imperfect explanation learning. 
Even when the selected explanatory structure is correct, the way it is instantiated as a perturbation can still move the target GNN away from its original representation regime. 
We therefore turn to the perturbation mechanism shared by existing perturb-query explainers.

\subsection{Element-wise Masking Induces Scale Drift}\label{sec:scale_drift}
Under deletion-oriented perturbations as defined in Eq.~\ref{eq:em}, reducing message scale is a natural consequence of suppressing an edge. 
The ambiguity arises when the resulting prediction change is interpreted directly as the information contribution of that edge, because EM jointly removes the edge-specific message signal and its contribution to propagation scale.

\paragraph{Scale Drift.} This effect is visible in Figure~\ref{fig:motivation}~(b,c). 
Under the same explanation, EM systematically shifts the message scale distribution toward smaller values compared with the original graph. 
As the contracted messages are repeatedly transformed and propagated across layers, the perturbed computation gradually departs from the original propagation regime. 
We term this progressive mismatch as \emph{Scale Drift}. 

Let $\widetilde{\mathcal S}_{\mathrm{EM}}^{(l)}(G)$ denotes the layer-$l$ message scale of the EM-perturbed computations. We denote $\eta_e^{(l)}$ as the fraction of layer-$l$ scale carried by the message induced by edge $e$:
\begin{equation}
    \eta_e^{(l)} :=
\|  m_e^{(l)}\|_2^2/
\sum_{a\in\mathcal E}\| m_a^{(l)}\|_2^2 .
\end{equation}
The following theorem  formalizes EM-induced Scale Drift.

\begin{theorem}[\textbf{Layer-wise Scale Contraction under EM}]
\label{theorem_em}
For any layer $l$, applying EM with coefficients $\{\rho_e\in[0,1]\}_{e\in\mathcal E}$ yields
\begin{equation}
\frac{S_{\mathrm{EM}}^{(l)}(G)}{S^{(l)}(G)}
\leq
1-\sum_{e\in\mathcal E}(1-\rho_e^2)\eta_e^{(l)}
\label{eq:em_contraction}
\end{equation}
The contraction is strict whenever there exists an edge e such that $\eta_e^{(l)}>0$ and $\rho_e < 1$.
\end{theorem}

\noindent \textit{Proof.} See Appendix B. Theorem~\ref{theorem_em} shows that EM introduces a deterministic scale
contraction whenever a message with nonzero scale is suppressed. Repeated
contractions across layers lead to the following depth-dependent result.

\begin{corollary}[\textbf{Depth-amplified Scale Drift under EM}]
\label{corollary_depth}
Under the conditions of Theorem~\ref{theorem_em}, suppose an edge $e$
satisfies $\rho_e\in(0,1)$ and
$\eta_e^{(l)}\geq\underline{\eta}>0$ for all $l=1,\ldots,L$. Then,
\begin{equation}
-\sum_{l=1}^{L}
\log\frac{S_{\mathrm{EM}}^{(l)}(G)}{S^{(l)}(G)}
\geq
-L\log\!\left[1-(1-\rho_e^2)\underline{\eta}\right].
\label{eq:depth_scale_drift}
\end{equation}
Hence, the drift grows at least linearly with deeper layers.
\end{corollary}

\noindent \textit{Proof.} See Appendix B. Corollary~\ref{corollary_depth} shows how repeated layer-wise contractions accumulate with depth, providing a mechanism-level explanation for the depth-amplified scale discrepancy observed in Figure~\ref{fig:motivation}(d).

\paragraph{Empirical Evidence.}
We empirically examine the above mechanism in Figure~\ref{fig:motivation}. 
Under the same oracle explanation, EM consistently shifts the message-norm distribution toward smaller values than the original graph as shown in (b) and (c). 
This local contraction further develops into a layer-wise scale gap.  
As shown in (d), the discrepancy between EM and the original graph becomes more pronounced across GNN layers which is consistent with the depth-amplified behavior. 
The same pattern also appears across datasets and backbones in (e), suggesting that Scale Drift is not an isolated artifact but a recurring consequence observed across the datasets.

Taken together, the oracle analysis, theoretical result, and empirical evidence show that EM introduces a confounding factor into perturb-query explanation. 
A prediction change under EM may not only reflect the removal of explanatory information, but also arise from the collapse of propagation scale. 
This ambiguity weakens the reliability of queried model responses, motivating a scale-stable perturbation mechanism that perturbs edge-induced message content without deterministically shrinking message scale. 
\definecolor{best}{HTML}{85C193}
\definecolor{second}{HTML}{D4EFDF}
\colorlet{best}{white}
\colorlet{second}{white}
\begin{table*}[!t]
  \centering
\resizebox{\textwidth}{!}{
\begin{tabular}{cc|cccccccc|c}
\toprule
\textbf{Metric} & \textbf{Method} & \textbf{Mutag} & \textbf{Benzene} & \textbf{Alkane} & \textbf{Fluoride} & \textbf{Indole} & \textbf{\textsc{PAINS}} & \textbf{R-Count} & \textbf{R-Max} & \textbf{Rank} \\
\midrule
\multirow{10}{*}{\textbf{F1}}
& Random
& \mstd{20.06}{0.55} & \mstd{30.80}{0.40} & \mstd{8.90}{1.09} & \mstd{24.52}{0.43} & \mstd{29.49}{0.10} & \mstd{29.00}{0.23} & \mstd{41.23}{0.08} & \mstd{34.95}{0.47} & 10 \\
& Saliency
& \mstd{49.00}{0.07} & \mstd{47.29}{0.01} & \mstd{0.00}{0.00} & \mstd{48.09}{0.00} & \mstd{50.03}{0.00} & \mstd{49.97}{0.01} & \mstd{37.97}{0.01} & \mstd{46.97}{0.00} & 6 \\
& GuidedBP
& \mstd{37.50}{0.00} & \mstd{45.31}{0.01} & \mstd{15.17}{0.00} & \mstd{41.68}{0.00} & \mstd{42.67}{0.00} & \mstd{51.66}{0.01} & \mstd{38.77}{0.01} & \mstd{47.03}{0.00} & 5 \\
& GNNExplainer
& \mstd{31.79}{1.90} & \mstd{25.88}{1.00} & \cellcolor{best}\textbf{\mstd{32.16}{2.21}} & \mstd{30.95}{0.49} & \mstd{42.22}{0.50} & \mstd{38.97}{0.04} & \mstd{51.50}{0.09} & \mstd{45.80}{1.01} & 8 \\
& PGExplainer
& \mstd{51.67}{0.24} & \mstd{64.96}{0.18} & \mstd{26.39}{0.00} & \cellcolor{second}\underline{\mstd{61.79}{0.40}} & \mstd{46.65}{0.03} & \mstd{45.98}{0.00} & \mstd{52.65}{0.00} & \cellcolor{second}\underline{\mstd{53.17}{0.00}} & 3 \\
& Refine
& \mstd{43.60}{0.82} & \mstd{42.88}{0.31} & \mstd{11.48}{0.00} & \mstd{30.36}{2.42} & \mstd{37.39}{0.06} & \mstd{54.18}{0.07} & \cellcolor{second}\underline{\mstd{53.73}{0.00}} & \mstd{51.67}{0.00} & 4 \\
& D4Explainer
& \mstd{30.68}{0.57} & \mstd{61.91}{0.22} & \mstd{26.58}{0.77} & \mstd{28.25}{1.02} & \mstd{36.29}{0.43} & \mstd{40.50}{0.10} & \mstd{50.96}{0.18} & \mstd{46.80}{1.38} & 9 \\
& ProxyExplainer
& \mstd{36.33}{0.00} & \mstd{49.52}{0.01} & \mstd{14.28}{0.00} & \mstd{38.22}{0.00} & \mstd{38.53}{0.12} & \mstd{37.18}{0.02} & \mstd{50.83}{0.00} & \mstd{51.29}{0.02} & 7 \\
& ConfExplainer
& \cellcolor{second}\underline{\mstd{52.19}{0.00}} & \cellcolor{second}\underline{\mstd{73.45}{0.01}} & \mstd{24.99}{0.00} & \mstd{49.97}{0.06} & \cellcolor{second}\underline{\mstd{57.82}{0.00}} & \cellcolor{second}\underline{\mstd{56.04}{0.00}} & \mstd{53.27}{0.00} & \mstd{51.97}{0.00} & 2 \\

& {\textbf{NICE (Ours)}}
& \cellcolor{best}\textbf{\mstd{53.90}{0.15}} & \cellcolor{best}\textbf{\mstd{78.49}{0.04}} & \cellcolor{second}\underline{\mstd{27.31}{0.12}} & \cellcolor{best}\textbf{\mstd{65.49}{0.05}} & \cellcolor{best}\textbf{\mstd{68.66}{0.00}} & \cellcolor{best}\textbf{\mstd{59.38}{0.01}} & \cellcolor{best}\textbf{\mstd{59.47}{0.00}} & \cellcolor{best}\textbf{\mstd{58.37}{0.00}} & 1 \\

\midrule
\multirow{10}{*}{\textbf{Recall}}
& Random
& \mstd{39.52}{1.33} & \mstd{32.71}{0.31} & \mstd{33.89}{4.29} & \mstd{35.55}{0.73} & \mstd{29.93}{0.13} & \mstd{29.88}{0.34} & \mstd{30.00}{0.06} & \mstd{40.11}{0.32} & 10 \\
& Saliency
& \mstd{90.13}{0.12} & \mstd{50.85}{0.02} & \mstd{0.00}{0.00} & \mstd{69.21}{0.00} & \mstd{53.34}{0.00} & \mstd{54.70}{0.01} & \mstd{27.67}{0.00} & \mstd{57.19}{0.00} & 5 \\
& GuidedBP
& \mstd{70.89}{0.00} & \mstd{48.52}{0.01} & \mstd{56.19}{0.00} & \mstd{57.89}{0.00} & \mstd{44.17}{0.00} & \mstd{53.26}{0.00} & \mstd{28.26}{0.00} & \mstd{56.94}{0.00} & 7 \\
& GNNExplainer
& \mstd{59.37}{3.67} & \mstd{50.71}{2.01} & \mstd{59.99}{4.20} & \mstd{59.65}{1.68} & \mstd{45.68}{0.37} & \mstd{41.29}{0.14} & \mstd{37.99}{0.40} & \mstd{53.24}{1.44} & 6 \\
& PGExplainer
& \cellcolor{second}\underline{\mstd{91.52}{0.29}} & \mstd{66.77}{0.24} & \cellcolor{second}\underline{\mstd{94.69}{0.00}} & \cellcolor{second}\underline{\mstd{84.97}{0.58}} & \mstd{49.43}{0.05} & \mstd{49.43}{0.05} & \mstd{37.89}{0.00} & \mstd{57.87}{0.00} & 4 \\
& Refine
& \mstd{40.24}{0.19} & \mstd{44.72}{0.28} & \mstd{89.38}{0.00} & \mstd{67.07}{1.80} & \mstd{35.91}{0.05} & \mstd{55.58}{0.08} & \cellcolor{second}\underline{\mstd{40.62}{0.01}} & \mstd{60.03}{0.00} & 3 \\
& D4Explainer
& \mstd{57.15}{1.18} & \cellcolor{best}\textbf{\mstd{96.83}{0.08}} & \mstd{93.95}{1.11} & \mstd{43.15}{1.49} & \mstd{35.30}{0.33} & \mstd{39.79}{0.16} & \mstd{35.40}{0.52} & \mstd{52.26}{0.92} & 8 \\
& ProxyExplainer
& \mstd{56.82}{0.00} & \mstd{52.24}{0.01} & \mstd{50.00}{0.00} & \mstd{56.42}{0.00} & \mstd{38.63}{0.07} & \mstd{40.36}{0.01} & \mstd{35.98}{0.00} & \mstd{57.06}{0.03} & 9 \\
& ConfExplainer
& \mstd{87.99}{0.00} & \mstd{76.38}{0.01} & \mstd{93.36}{0.00} & \mstd{67.91}{0.01} & \cellcolor{second}\underline{\mstd{59.87}{0.00}} & \cellcolor{second}\underline{\mstd{57.56}{0.01}} & \mstd{40.26}{0.00} & \cellcolor{second}\underline{\mstd{61.17}{0.00}} & 2 \\

& {\textbf{NICE (Ours)}}
& \cellcolor{best}\textbf{\mstd{93.92}{0.12}} & \cellcolor{second}\underline{\mstd{82.60}{0.08}} & \cellcolor{best}\textbf{\mstd{97.64}{0.51}} & \cellcolor{best}\textbf{\mstd{88.50}{0.05}} & \cellcolor{best}\textbf{\mstd{76.56}{0.00}} & \cellcolor{best}\textbf{\mstd{61.93}{0.00}} & \cellcolor{best}\textbf{\mstd{44.26}{0.01}} & \cellcolor{best}\textbf{\mstd{68.42}{0.00}} & 1 \\

\midrule
\multirow{10}{*}{\textbf{AUC}}
& Random
& \mstd{55.47}{0.72} & \mstd{51.94}{0.24} & \mstd{52.05}{2.24} & \mstd{53.33}{0.40} & \mstd{49.94}{0.08} & \mstd{49.97}{0.21} & \mstd{50.01}{0.11} & \mstd{50.47}{0.19} & 10 \\
& Saliency
& \mstd{84.86}{0.02} & \mstd{64.17}{0.01} & \mstd{34.32}{0.00} & \mstd{73.46}{0.00} & \mstd{66.32}{0.00} & \mstd{67.33}{0.01} & \mstd{45.95}{0.00} & \mstd{67.37}{0.00} & 4 \\
& GuidedBP
& \mstd{73.85}{0.00} & \mstd{62.59}{0.01} & \mstd{63.76}{0.00} & \mstd{66.99}{0.00} & \mstd{60.33}{0.00} & \mstd{66.93}{0.00} & \mstd{46.85}{0.01} & \mstd{67.17}{0.00} & 8 \\
& GNNExplainer
& \mstd{66.84}{2.07} & \mstd{61.68}{1.10} & \mstd{67.22}{2.37} & \mstd{66.78}{1.00} & \mstd{60.92}{0.27} & \mstd{58.12}{0.08} & \mstd{62.97}{0.17} & \mstd{65.33}{0.83} & 9 \\

& PGExplainer
& \cellcolor{second}\underline{\mstd{85.55}{0.17}} & \mstd{76.40}{0.14} & \mstd{84.03}{0.00} & \cellcolor{second}\underline{\mstd{83.50}{0.35}} & \mstd{63.84}{0.03} & \mstd{64.44}{0.00} & \mstd{63.53}{0.00} & \mstd{69.15}{0.00} & 3 \\

& Refine
& \mstd{56.38}{0.13} & \mstd{60.44}{0.19} & \mstd{60.18}{0.00} & \mstd{69.99}{1.04} & \mstd{54.79}{0.04} & \mstd{68.90}{0.05} & \mstd{65.65}{0.03} & \mstd{69.91}{0.00} & 6 \\
& D4Explainer
& \mstd{65.40}{0.66} & \mstd{71.36}{0.08} & \cellcolor{second}\underline{\mstd{83.68}{0.62}} & \mstd{57.74}{0.85} & \mstd{64.08}{0.24} & \mstd{57.36}{0.10} & \mstd{63.25}{0.66} & \mstd{65.09}{0.69} & 7 \\

& ProxyExplainer
& \mstd{66.44}{0.00} & \mstd{65.38}{0.00} & \mstd{60.54}{0.00} & \mstd{65.64}{0.00} & \cellcolor{second}\underline{\mstd{76.43}{0.06}} & \mstd{57.26}{0.01} & \mstd{63.38}{0.00} & \mstd{68.18}{0.02} & 5 \\

& ConfExplainer
& \mstd{84.24}{0.00} & \cellcolor{second}\underline{\mstd{82.74}{0.01}} & \mstd{83.25}{0.00} & \mstd{73.34}{0.01} & \mstd{71.66}{0.00} & \cellcolor{second}\underline{\mstd{69.93}{0.01}} & \cellcolor{second}\underline{\mstd{75.49}{0.01}} & \cellcolor{second}\underline{\mstd{70.71}{0.00}} & 2 \\

& {\textbf{NICE (Ours)}}
& \cellcolor{best}\textbf{\mstd{87.08}{0.11}} & \cellcolor{best}\textbf{\mstd{86.62}{0.04}} & \cellcolor{best}\textbf{\mstd{85.59}{0.27}} & \cellcolor{best}\textbf{\mstd{85.86}{0.03}} & \cellcolor{best}\textbf{\mstd{82.94}{0.00}} & \cellcolor{best}\textbf{\mstd{76.49}{0.01}} & \cellcolor{best}\textbf{\mstd{83.17}{0.01}} & \cellcolor{best}\textbf{\mstd{84.15}{0.00}} & 1 \\
\bottomrule
\end{tabular}
}
\caption{
\textbf{Explanation performance of NICE.}
We retain the Top30\% edges and compare them against ground-truth edges.
Results are reported as Mean $\pm$ Standard Deviation.
The \colorbox{best}{\textbf{best}} is bold, and the \colorbox{second}{\underline{second-best}} is underlined.
}
\vspace{-0.3cm}
\label{tab:main_results}
\end{table*}

\section{Methodology}\label{sec:method}

In this section, We present NICE, a Noise Corruption-based framework that reformulates edge explanation as a stochastic restoration process. 
Rather than suppressing edge-induced messages toward zero, NICE starts from matched-norm corrupted counterparts and examines how clean-message directions should be restored to recover the original target prediction. 
Procedurally, NICE consists of three components. 
First, it defines a scale-stable perturbation space through \emph{Noise Corruption} (NC), which replaces multiplicative suppression with matched-norm stochastic corruption. 
Second, \emph{Stochastic Restoration Boundary Learning} (SRB) interprets the NC coefficient as a restoration gate and learns a compact boundary that preserves the target prediction under stochastic corruption. 
Third, since this boundary specifies a restoration configuration, \emph{Boundary-Integrated Gradients} (BIG) decodes it into importance scores by measuring the element-wise contribution along the restoration path. 

\subsection{Scale-Stable Noise Corruption}
\label{sec:nc}
The diagnosis above suggests that a desirable alternative should perturb the information carried by edge-induced messages without deterministically shrinking
their scale defined in Eq.~\eqref{eq:scale}. 
To this end, NICE introduces \emph{Noise Corruption} (NC), which replaces zero-directed suppression with matched-norm random-direction corruption. 
For the message $\mathbf m_e^{(l)}\in\mathbb R^{d_l}$ induced by edge $e$ at layer $l$,  we sample a corrupted counterpart from the sphere with the same norm:
\begin{equation} 
\boldsymbol\epsilon_e^{(l)} \sim \mathcal P_e^{(l)}:= 
\mathrm{Unif} \left( \left\{ \mathbf z\in\mathbb R^{d_l}: \|\mathbf z\|_2=\|\mathbf m_e^{(l)}\|_2 \right\} \right). 
\label{eq:nc_noise} 
\end{equation} 
We use $\mathcal P_{\mathrm{NC}}$ to denote the joint distribution of all corrupted counterparts sampled across edges and message-passing layers in one NC forward pass. 
Given a restoration gate $\rho_e\in[0,1]$, NC corrupts the message with the noise:
\begin{equation}
\Psi_{\mathrm{NC},\rho_e}^{(l)} \left( \mathbf m_e^{(l)} \right) =
\sqrt{\rho_e}\mathbf m_e^{(l)} + \sqrt{1-\rho_e}\boldsymbol\epsilon_e^{(l)}.
\label{eq:nc}
\end{equation}
Here, $\rho_e$ controls how far the message is restored toward its clean direction: $\rho_e=0$ gives a fully corrupted counterpart, whereas $\rho_e=1$ recovers the clean message. 
Therefore, Varying $\rho_e$ defines a continuous restoration path within the matched-norm corruption space. 
The proposition shows that the restoration path satisfies the scale-stability requirement.

\begin{proposition}[\textbf{Scale-Stable Restoration Path under NC}]
\label{prop:nc_scale_stability}
% \begin{proposition}[Scale Stability of NC]
For any fixed message $\mathbf m_e^{(l)}$ and restoration gate
$\rho_e\in[0,1]$,
\begin{equation}
\mathbb E_{\boldsymbol\epsilon_e^{(l)}}
\left[
\left\|
\Psi_{\mathrm{NC},\rho_e}^{(l)}
\left(\mathbf m_e^{(l)}\right)
\right\|_2^2
\right]
=
\left\|\mathbf m_e^{(l)}\right\|_2^2.
\end{equation}
\end{proposition}

\noindent\textit{Proof.}
The result follows by expanding Eq.~\eqref{eq:nc} and the full proof is provided in Appendix B. 
By preserving message scale in expectation, NC excludes the systematic scale-shrinkage factor introduced by EM, allowing queried prediction changes to be more directly attributed to the perturbation of edge-induced message information. 
However, NC alone does not specify an explanation. 
What remains is to identify a compact configuration of restoration gates whose queried behavior remains sufficiently close to the original target decision under stochastic noise corruption.

\subsection{Stochastic Restoration Boundary Learning}\label{sec:srb}

Building on the NC restoration path, we formulate the above problem as \emph{Stochastic Restoration Boundary Learning} (SRB). 
For the input graph $G=(\mathcal V,\mathcal E)$ and its target prediction $y^*$, we seeks a joint restoration configuration $ \boldsymbol{\rho}=\{\rho_e\}_{e\in E} $. 
As defined in Eq.~\eqref{eq:nc}, each \(\rho_e\) controls how far the corresponding edge-induced message is restored from its corrupted counterpart toward its clean direction. 
The objective is to restore only the clean-message directions necessary for maintaining the original target prediction.

\paragraph{Restoration Sufficiency under NC.}
For a fixed restoration configuration \(\boldsymbol{\rho}\), the queried prediction remains stochastic due to each NC forward pass draws corrupted counterparts from \(\mathcal P_{\mathrm{NC}}\). 
For one NC samples \(\boldsymbol{\epsilon}\sim\mathcal P_{\mathrm{NC}}\), let \(f(G;\boldsymbol{\rho},\boldsymbol{\epsilon})\) denote the corresponding target-class probability. 
We measure how well \(\boldsymbol{\rho}\) restores the target predictions through We the one-sided target-prediction degradation:
\begin{equation} 
d_{y^*}(\boldsymbol\rho,\boldsymbol\epsilon)
=
\left[
\log f_{y^*}(G)
-
\log f_{y^*}(\boldsymbol\rho,\boldsymbol\epsilon)
\right]_+.
\label{eq:srb_behavior_deviation} \end{equation}
The deviation is zero when the corrupted computation preserves or strengthens the target-class probability, and increases only when support for the original prediction degrades.
It therefore evaluates restoration sufficiency with respect to the target decision without requiring the complete queried output to reproduce the clean prediction.

The same restoration configuration may induce different deviations across
noise realizations. We account for both their average magnitude and
sampling-induced variation through the stochastic restoration error
\begin{equation} 
\mathcal R_{y^*}(\boldsymbol{\rho}) = \mathbb E_{\boldsymbol{\epsilon}} \left[ d_{y^*}(\boldsymbol{\rho},\boldsymbol{\epsilon}) \right] + \beta\, \operatorname{Std}_{\boldsymbol{\epsilon}} \left[ d_{y^*}(\boldsymbol{\rho},\boldsymbol{\epsilon}) \right], 
\label{eq:srb_restoration_risk} 
\end{equation}
where \(\beta\geq0\) controls the penalty on corruption-induced variation.
A small \(\mathcal R(\boldsymbol{\rho})\) indicates that the target
behavior is restored both accurately and consistently under NC.

\paragraph{Restoration boundary learning.}
SRB balances stochastic restoration risk against restoration compactness by:
\begin{equation}
    \mathcal L_{\mathrm{SRB}}(\boldsymbol{\rho})
    =
    \mathcal R_{y^*}(\boldsymbol{\rho})
    +
    \lambda_{\mathrm{rest}}\|\boldsymbol{\rho}\|_1,
    \label{eq:srb_objective}
\end{equation}
where $\lambda_{\mathrm{rest}}\geq 0$ controls the trade-off between
restoring the target prediction and retaining a compact restoration
configuration. The resulting stochastic restoration boundary is
\begin{equation}
    \boldsymbol{\rho}^{\star}
    =
    \arg\min_{\boldsymbol{\rho}\in[0,1]^{|\mathcal E|}}
    \mathcal L_{\mathrm{SRB}}(\boldsymbol{\rho}).
    \label{eq:srb_solution}
\end{equation}

In practice, we estimate $\mathcal R_{y^*}(\boldsymbol{\rho})$ using
$N$ independent NC samples. Let
$\boldsymbol{\epsilon}^{(n)}\sim\mathcal P_{\mathrm{NC}}$ and
$d_n=d_{y^*}(\boldsymbol{\rho},\boldsymbol{\epsilon}^{(n)})$ for
$n=1,\ldots,N$. The empirical restoration risk is
\begin{equation}
    \widehat{\mathcal R}_{y^*}^{(N)}(\boldsymbol{\rho})
    =
    \bar d_N
    +
    \beta
    \sqrt{
        \frac{1}{N}
        \sum_{n=1}^{N}(d_n-\bar d_N)^2
    },
    \bar d_N=\frac{1}{N}\sum_{n=1}^{N}d_n .
    \label{eq:empirical_restoration_risk}
\end{equation}
At each optimization step, we resample the $N$ noise realizations and
optimize the empirical counterpart of Eq.~\ref{eq:srb_objective} by backpropagating through
$\widehat{\mathcal R}_{y^*}^{(N)}(\boldsymbol{\rho})$.

\subsection{Boundary-Integrated Gradient}
\label{sec:big}

Although SRB yields a compact restoration boundary $\boldsymbol{\rho}^{\star}$, each $\rho_e^{\star}$ only records how much the clean-message direction of edge $e$ is restored under the global restoration objective. 
To bridge this gap, we propose \emph{Boundary-Integrated Gradient} (BIG), which integrates this contribution from the fully corrupted state to the learned boundary along:
\begin{equation} 
\boldsymbol{\rho}(t) = t\boldsymbol{\rho}^{\star}, \qquad t\in[0,1]. 
\label{eq:big_path} 
\end{equation} 
The attribution score of edge \(e\) is then computed as 
\begin{equation} 
s_e^{\mathrm{BIG}} = -\rho_e^{\star} \int_0^1 \left. \frac{\partial\mathcal R_{y^*}(\boldsymbol{\rho})} {\partial\rho_e} \right|_{\boldsymbol{\rho}=\boldsymbol{\rho}(t)} \mathrm dt. 
\label{eq:big_score} 
\end{equation} 
The gradient measures how restoring edge \(e\) changes the stochastic restoration risk, while the integration accumulates this effect along the restoration path. 
In practice, we approximate the integral in Eq.~\ref{eq:big_score} using a Riemann sum over $T$ uniformly spaced points along the restoration path, where $T$ denotes the number of integration steps.

\section{Experimental Study}

We design our experiments to answer following questions. 
\textbf{RQ1:} How effective is NICE in terms of agreement with ground-truth and model faithfulness?
\textbf{RQ2:} How do EM and NC differ in their effects on message scale and the resulting explanations? 
\textbf{RQ3:} Under stochastic NC, can SRB learn a compact restoration boundary that restores the target behavior consistently across noise realizations?
\textbf{RQ4:} Is BIG necessary for decoding the learned restoration boundary?

\begin{figure*}[tbp]
    \centering
    \includegraphics[width=1\linewidth]{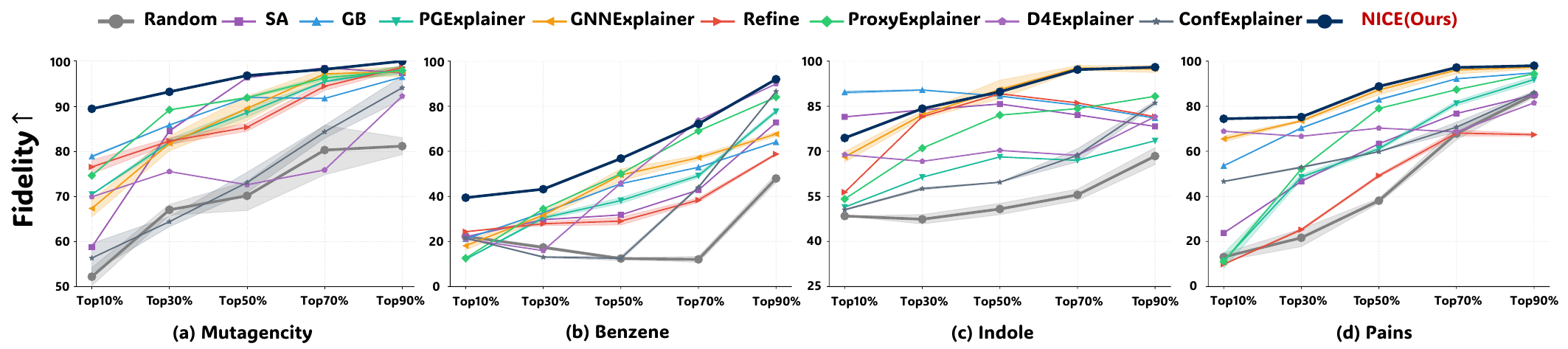}
\caption{
\textbf{Fidelity under different sparsity budgets.}
We vary the explanation sparsity from Top10\% to Top90\% and measure the resulting Fidelity. 
Higher values indicate that the explanation better preserves the target model prediction. 
}
    \label{fig:fidelity_curve} \vspace{-0.2cm}
\end{figure*}

\subsection{Experimental Setup}
\textbf{Datasets.} 
Following prior work \cite{vinfor}, we We evaluate on eight graph-classification benchmarks with annotated explanatory substructures, including four widely used molecular datasets namely Mutagenicity (Mg),  Benzene (Bz), Alkane-Carbonyl (AC) and Fluoride-Carbonyl (FC) \cite{graphxai}.
Additionally include seven tasks from the more challenging \textsc{B-XAIC} benchmark, namely Indole, PAINS, Rings-Count and Rings-Max \cite{B-XAIC}. 
Details are provided in Appendix C. 

\noindent \textbf{Baselines.} 
Three types of baselines are adopted. 
(1) \emph{Reference explainers}: GTExplainer which returns the ground-truth explanations as an oracle reference and Random explains randomly as a weak reference. 
(2) \emph{White-box explainers}: gradient-based methods including Saliency~\cite{SA} and GuidedBackprop~\cite{GB}. 
(3) \emph{Black-box explainers}: methods follow the Perturb-Query paradigm, including GNNExplainer~\cite{survey-counter}, PGExplainer~\cite{burkart2021survey}, ReFine~\cite{ReFine}, D4Explainer~\cite{d4e}, ProxyExplainer~\cite{proxy}, ConfExplainer~\cite{Zhang2025IsYE}. 
More details see Appendix D.

\noindent \textbf{Evaluation Protocol.} 
We train a GIN classifier as the target model for graph-classification tasks and explain its predictions. 
We implement explanation evaluation align with IDEA \cite{yin2026paradig} and reporting Precision, Recall, F1 and AUC-ROC. 
To assess whether explanations remain faithful to the target model, we further report the widely used Fidelity \cite{robust_f, global_concept_ex}. 

\noindent \textbf{Implementation Details.}
Following PGExplainer \cite{pge}, we parameterize each restoration gate with an MLP over its edge representation and obtain $\rho_e\in[0,1]$. 
The  parameters are optimized using Adam with a learning rate of $0.001$. 
By default, we use $16$ NC samples to estimate the restoration risk, set $\beta=0.1$ and $\lambda_{\mathrm{rest}}=1.0$ and approximate BIG with $16$ integration steps. Results are averaged over $5$ random seeds. Details are provided in Appendix E.

\begin{figure*}[tbp]
    \centering
    \includegraphics[width=1\linewidth]{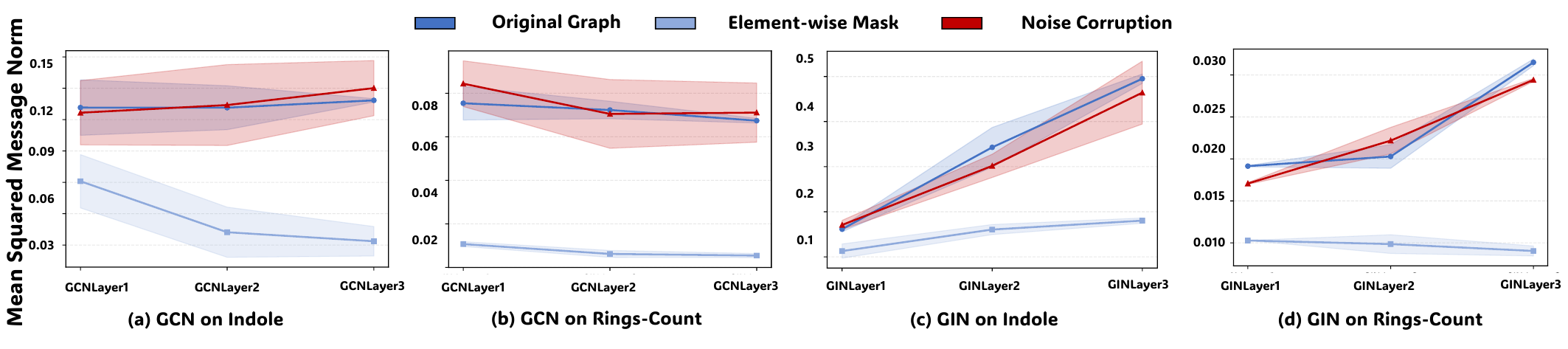}
\caption{
\textbf{Layer-wise message scale under the same GTExplainer configuration.} EM causes persistent message-scale contraction, whereas NC remains closer to the clean computation across datasets and GNN backbones.
}
    \label{fig:exp_fig_2} \vspace{-0.2cm}
\end{figure*}

\subsection{Overall Explanation Effectiveness (RQ1)}
We evaluate NICE from two complementary perspectives.

\begin{table}[htbp]
\centering
\label{tab:primitive_control}
\resizebox{\columnwidth}{!}{
\begin{tabular}{c c c c c c}
\toprule
& & \multicolumn{2}{c}{$D_{\mathrm{repr}}\downarrow$}
& \multicolumn{2}{c}{$D_{\mathrm{pred}}\downarrow$} \\
\cmidrule(lr){3-4}
\cmidrule(lr){5-6}
Dataset & Method & GT & Rand & GT & Rand \\
\midrule
\multirow{3}{*}{R-Count}
& EM & \mstd{0.31}{0.09} & \mstd{0.63}{0.22} & \mstd{0.57}{0.15} & \mstd{0.59}{0.12}\\
& NC & \mstd{0.07}{0.02} & \mstd{0.08}{0.02} & \mstd{0.02}{0.01} & \mstd{0.04}{0.04}\\
& Imp. & 77.42\% & 87.30\% & 96.49\% & 93.22\%  \\
\midrule
\multirow{3}{*}{R-Max}
& EM & \mstd{0.45}{0.10} & \mstd{0.61}{0.11} & \mstd{0.36}{0.16} & \mstd{0.48}{0.18}\\
& NC & \mstd{0.16}{0.03} & \mstd{0.16}{0.04} & \mstd{0.14}{0.05} & \mstd{0.13}{0.06}\\
& Imp. & 64.44\% & 73.77\% & 61.11\% & 72.92\%  \\
\bottomrule
\end{tabular}
}
\caption{
\textbf{Controlled comparison of EM and NC under GTExplainer and Random.} 
$D_{\mathrm{repr}}$ and $D_{\mathrm{pred}}$ measure representation distance and target-prediction normalized degradation. 
\textit{Imp.} denotes the relative reduction ratio.
}
\vspace{-0.2cm}
\label{table_2}
\end{table}

\noindent \textbf{Agreement with Ground-Truth Explanations.}
Table~\ref{tab:main_results} shows that NICE ranks first on F1, Recall, and AUC-ROC, outperforming the strongest baseline by an average of $6.42\%$, $8.67\%$, and $7.57\%$ respectively. 
Consistent gains are observed regardless of whether tasks are easy or challenging. 
In particular, the higher Recall and AUC indicate that NICE improves both the coverage and ranking of ground-truth edges, rather than relying on a favorable selection threshold.

\noindent \textbf{Faithfulness under Varying Sparsity.} 
Figure~\ref{fig:fidelity_curve} shows that NICE is particularly advantageous under small explanation budgets. 
This indicates that its highest-ranked edges preserve most of the target prediction with only a compact subgraph. 
As more edges are retained, the performance gap narrows because large budgets reduce the influence of ranking quality.

\subsection{Diagnosing Scale Drift and Evaluating NC (RQ2)}

We use the same explanations to examine the effects on message scale, graph representations and queried predictions.

% We first instantiate the same benchmark-annotated edge configuration using EM and NC, thereby isolating the effect of the perturbation primitive from that of edge selection. 
% As shown in Figure~\ref{fig:exp_fig_2}, EM systematically contracts message scale from the first layer, and the discrepancy from the clean computation persists or becomes more pronounced across subsequent message-passing layers. 
% In contrast, NC closely follows the clean message scale across representative datasets and both GCN and GIN backbones. 
% These results directly validate that EM introduces layer-wise Scale Drift, while NC substantially alleviates its deterministic scale contraction.
\noindent \textbf{Depth-amplified Scale Drift under EM.} 
As shown in Figure~\ref{fig:exp_fig_2}, EM contracts message scale from the first layer and this discrepancy persists or widens through subsequent message-passing layers. 
This pattern appears across representative datasets and both GCN and GIN backbones, confirming that Scale Drift is a systematic consequence of multiplicative masking rather than an isolated perturbation at a single layer. 
In contrast, NC remains substantially closer to the clean message-scale trajectory, validating its scale-stable ability.

\begin{figure}[tbp]
    \centering
    \includegraphics[width=1\linewidth]{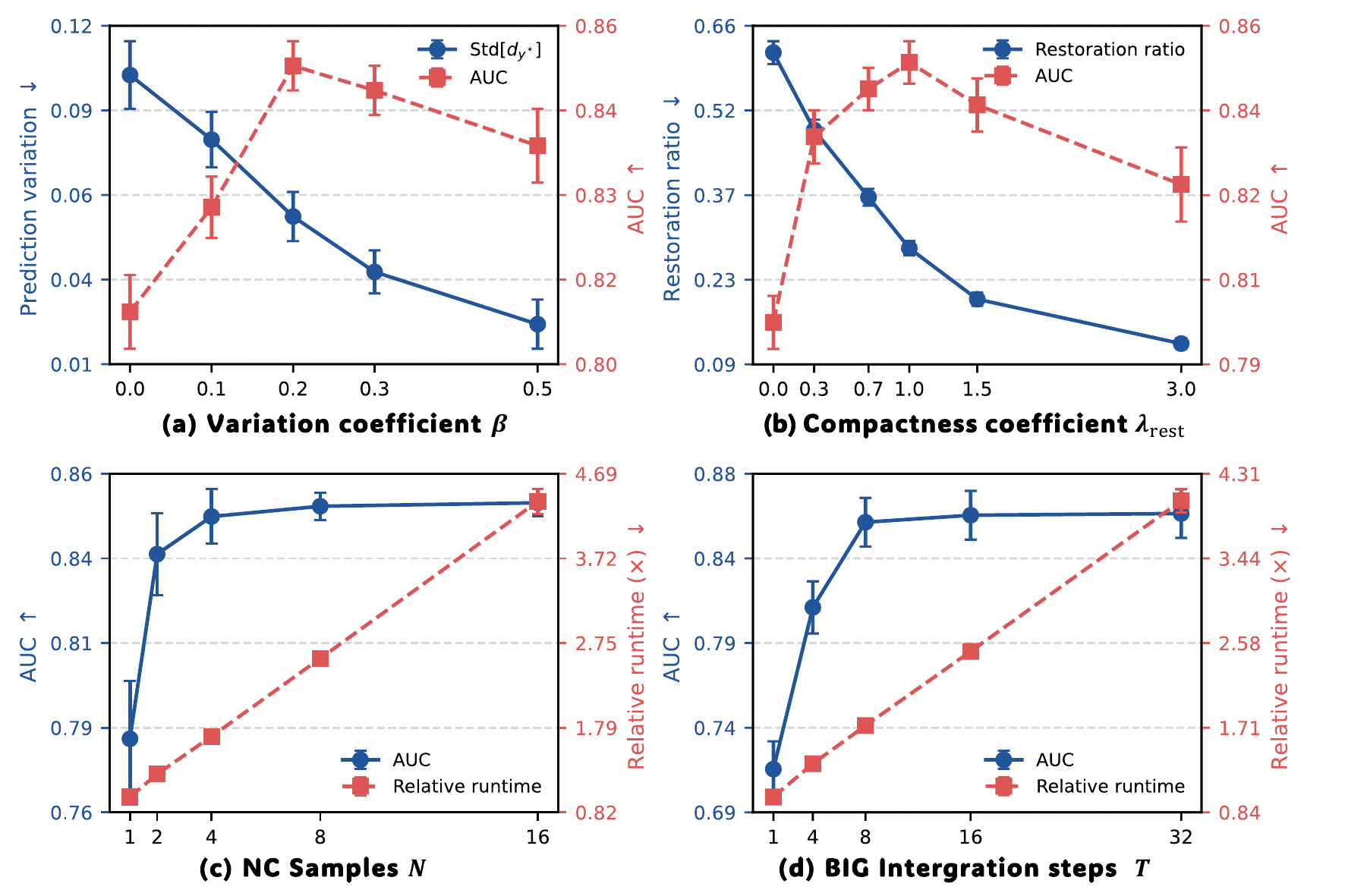}
\caption{
\textbf{Hyperparameter sensitivity of SRB and BIG.}
We study the effects of the variation coefficient $\beta$, compactness coefficient $\lambda_{\mathrm{rest}}$, NC sample number $N$, and BIG integration steps $T$ on performance, stability, compactness, and runtime. 
}
    \label{fig:exp_fig_3} \vspace{-0.3cm}
\end{figure}

\noindent\textbf{Representation and prediction effects.}
Table~\ref{table_2} further compares EM and NC under ground-truth and sparsity-matched random scores.  
NC reduces the normalized representation distance by $70.9\%$ and the target-prediction degradation by $78.8\%$. 
Importantly, these improvements persist under random scores, indicating that they arise from the perturbation mechanism instead of explanation quality alone. 
Together, the results show that EM-induced scale drift is accompanied by larger representation and prediction deviations, while NC keeps the perturbed computation closer to the clean regime.

\subsection{Ablation Study (RQ3 \& RQ4)}

We examine whether SRB learns compact restoration boundary and how BIG converts them into attribution scores.

\noindent \textbf{Why SRB is necessary.} 
Table~\ref{srb} verifies the dual objectives of SRB. 
Removing the variation penalty increases the mean prediction degradation by $61\%$ and its standard deviation by $217\%$, showing that variation aware
optimization boosts restoration quality and stability across NC samples. 
Without compactness regularization, the restoration ratio increases by $211\%$, while AUC decreases by $6.9\%$. 
Thus, restoring more messages can not yield superior explanations; compactness is essential for preventing over-restored boundary.

\noindent \textbf{From SRB to BIG.}
Figure~\ref{fig:case} compares the learned restoration boundary $\boldsymbol\rho^\star$ with $s^{\mathrm{BIG}}$ on the same graph. 
While each $\rho^\star$ records how far edge $e$ is restored at the learned boundary. 
BIG instead accumulates the contribution of the edge to reducing stochastic restoration risk along the restoration path, yielding a sharper ranking over the ground-truth explanatory edges. 
This comparison illustrating how it converts the joint restoration boundary into edge-level attributions.

\noindent \textbf{Sensitivity and Efficiency.}
% Figure~\ref{fig:exp_fig_3}(a--b) reveals clear trade-offs in SRB: increasing $\beta$ improves stability across NC samples, while a larger $\lambda_{\mathrm{rest}}$ yields a more compact boundary; overly large values, however, reduce attribution quality or prediction restoration. 
% Figure~\ref{fig:exp_fig_3}(c--d) shows that performance quickly saturates as the numbers of NC samples and BIG integration steps increase, whereas runtime continues to grow, supporting the default settings as a favorable accuracy--efficiency trade-off.
Figure~\ref{fig:exp_fig_3}(a--b) shows clear trade-offs in SRB. 
Increasing $\beta$ improves stability across NC samples, whereas an excessively large value reduces attribution quality. 
Similarly, increasing $\lambda_{\mathrm{rest}}$ yields a more compact boundary but eventually compromises target-prediction restoration. 
Figure~\ref{fig:exp_fig_3}(c--d) shows that performance quickly saturates as the numbers of NC samples and BIG integration steps increase, whereas runtime continues to grow, supporting the default settings as a favorable trade-off.

\begin{figure}[tbp]
    \centering
    \includegraphics[width=1\linewidth]{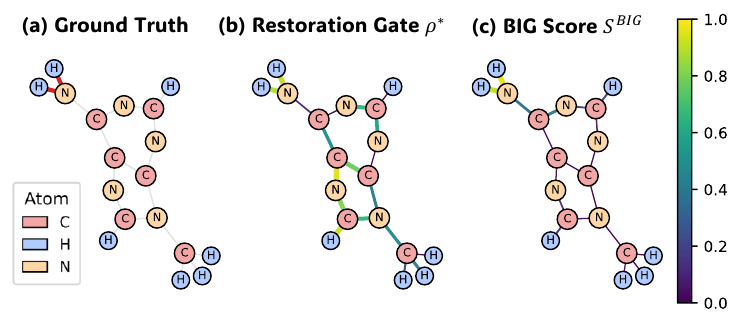}
\caption{
\textbf{Case study.} 
Gate scores indicate restoration degrees, whereas BIG more clearly highlights edges.
}
    \label{fig:case} \vspace{-0.3cm}
\end{figure}

\begin{table}[t]
\centering
\label{tab:srb_big}
\resizebox{\columnwidth}{!}{
\begin{tabular}{ccccc}
\toprule
Method
& Mean $\downarrow$
& Std. $\downarrow$
& Rest.$\downarrow$
& AUC $\uparrow$  \\
\midrule
% \multicolumn{6}{l}{\textit{(a) Restoration boundary learning}} \\
Full SRB
& \textbf{\mstd{0.23}{0.06}} & \textbf{\mstd{0.06}{0.02}} & \textbf{\mstd{0.27}{0.11}} & \textbf{\mstd{0.87}{0.00} } \\
w/o Var.
& \mstd{0.37}{0.08} & \mstd{0.19}{0.07} & \underline{\mstd{0.41}{0.13}} & \underline{\mstd{0.83}{0.06}}  \\
w/o Comp.
& \underline{\mstd{0.25}{0.11}} & \underline{\mstd{0.09}{0.04}} & \mstd{0.84}{0.42} & \mstd{0.81}{0.02}  \\
\bottomrule
\end{tabular}
}
\caption{
Analysis of SRB. Results are averaged over Benzene.
}
\label{srb}
\vspace{-0.3cm}
\end{table}
\section{Conclusion}

This work revisits distribution shift in Perturb-Query GNN explanation from the perspective of the perturbation mechanism. 
We identify \textit{Scale Drift}, showing that Element-wise Masking couples information corruption with deterministic scale contraction that accumulates across
message-passing layers.
To avoid this confounding effect, we introduce Noise Corruption, which defines a scale-stable restoration space through matched-norm random-direction corruption.  
Building on this space, NICE first learns a compact stochastic restoration boundary through SRB and then converts the restoration process into explanation scores through BIG. 
Experiments demonstrate stronger performance while validating the Scale Drift diagnosis and the scale-stability of NC.

% Several limitations remain.
% NC preserves the squared message norm only in expectation,
% rather than for every sample or for the complete message
% distribution.
% NICE also requires differentiable access to internal message
% passing and incurs additional computation from NC sampling
% and path integration.
% Moreover, our current analysis focuses on edge-level
% explanations for graph classification, and the theoretical
% characterization relies on local linearization.
% Extending NICE to other explanation settings and developing
% more efficient stochastic estimators are promising directions.

\bibliography{aaai2027, iclr2026_conference, sample-base}

\appendix
% Required packages

\clearpage

\section{A. Notation}
\label{sec:notation}

The main notation is summarized in Table~\ref{tab:notation_1} and Table~\ref{tab:notation_2}. 

\begin{table}[htbp]
    \centering
    \renewcommand{\arraystretch}{1.05}
    \setlength{\tabcolsep}{3pt}

    \resizebox{\columnwidth}{!}{
    \begin{tabular}{
        @{}
        >{\centering\arraybackslash}p{2.7cm}
        >{\raggedright\arraybackslash}p{6.15cm}
        @{}
    }
        \toprule
        \textbf{Symbol} & \textbf{Description} \\
        \midrule

        \multicolumn{2}{@{}l}{\textbf{Graph and target GNN}}\\[-1pt]
        $G=(\mathcal V, \mathcal E)$ & Input graph with node set $\mathcal{V}$ and edge set $\mathcal{E}$. \\
        $\mathcal V=\{v_1,\ldots,v_n\}$ & Set of nodes in graph $G$. \\
        $\mathcal E \subseteq \mathcal{V} \times \mathcal{V}$ & Set of edges in graph $G$. \\
        $v_i$ & The $i$-th node in graph $G$. \\
        $e_{ij}=(v_i,v_j)$ & Edge between nodes $v_i$ and $v_j$. \\
        % $n=|V|$ & Number of nodes in graph $G$. \\
        % $|E|$ & Number of edges in graph $G$. \\ 
        $C$ & Number of prediction classes. \\
        % $c$ & Index of a prediction class. \\
        $L$ & Number of message-passing layers in GNN. \\
        $l$ & Index of a layer, $l\in\{1,\ldots,L\}$. \\
        $d_l$ & Dimension of a message at layer $l$. \\
        $\mathbf{h}_i^{(l)}$ & Representation of node $v_i$ at layer $l$. \\
        $\phi^{(l)}$ & Message function at layer $l$. \\
        $\mathbf{m}_{i\leftrightarrow j}^{(l)}$ & Message between nodes $v_i$ and $v_j$ at layer $l$. \\
        $\mathbf{m}_e^{(l)}$ & Clean message induced by edge $e$ at layer $l$. \\
        $f$ & Fixed target GNN being explained. \\
        $f(G)\in\mathbb{R}^{C}$ & Prediction of the target GNN on graph $G$. \\
        $f_c(G)$ & Predicted score or probability for class $c$. \\
        $y^\star$ & Target class explained by explainer. \\
        \addlinespace[1pt]
        \midrule
        \multicolumn{2}{@{}l}{\textbf{Explanation and perturbation}}\\[-1pt]

        $s_e\in[0,1]$ & Attribution score assigned to edge $e$. \\
        $\mathbf{s}=\{s_e\}_{e\in E}$ & Collection of edge-attribution scores. \\
        $\rho_e\in[0,1]$ & Intervention coefficient over edge $e$. \\
        $\boldsymbol{\rho} =\{\rho_e\}_{e\in E}$ & Joint intervention/restoration configuration. \\
        $\Psi_{\rho_e}^{(l)}$ & Generic perturbation operator at layer $l$. \\
        $\widetilde{\mathbf{m}}_e^{(l)}$ & Perturbed message induced by $e$ at layer $l$. \\
        $\Psi_{\mathrm{EM},\rho_e}^{(l)}$& Element-wise Masking controlled by $\rho_e$. \\
        $\Psi_{\mathrm{NC},\rho_e}^{(l)}$& Noise Corruption operator controlled by $\rho_e$. \\

        \addlinespace[1pt]
        \midrule
        \multicolumn{2}{@{}l}{\textbf{Message scale and Scale Drift}}\\[-1pt]

        $\mathcal{S}^{(l)}(G)$& Layer-wise message scale  at layer $l$. \\
        $\mathcal{S}_{\mathrm{EM}}^{(l)}(G)$& Layer-wise message scale under EM. \\
        $\mathcal{S}_{\mathrm{NC}}^{(l)}(G)$& Layer-wise message scale under NC. \\
        $\eta_e^{(l)}$& Fraction of message scale carried by edge $e$. \\
        $\mathcal{S}_{\mathrm{EM}}^{(l)}(G)/\mathcal{S}^{(l)}(G)$ & Relative scale retained under EM. \\

        \bottomrule
    \end{tabular}
    }
    \caption{Summary of notation (Part 1).}
    \label{tab:notation_1}
\end{table}

\section{B. Proofs}
\label{app:proofs}

\begin{table}[ht]
    \centering
    \renewcommand{\arraystretch}{1.05}
    \setlength{\tabcolsep}{3pt}

    \resizebox{\columnwidth}{!}{
    \begin{tabular}{
        @{}
        >{\centering\arraybackslash}p{2.35cm}
        >{\raggedright\arraybackslash}p{6.15cm}
        @{}
    }
        \toprule
        \textbf{Symbol} & \textbf{Description} \\
        \midrule
        \multicolumn{2}{@{}l}{\textbf{Noise Corruption}}\\[-1pt]
        $\boldsymbol{\epsilon}_e^{(l)}$ & Matched-norm random corruption of $\mathbf{m}_e^{(l)}$. \\
        $\mathcal{P}_e^{(l)}$ & Distribution of the corrupted message $\boldsymbol{\epsilon}_e^{(l)}$. \\
        $\boldsymbol{\epsilon}$& Collection of corruption variables sampled. \\
        $\mathcal{P}_{\mathrm{NC}}$ & Joint distribution of corruption variables. \\
        $f(G;\boldsymbol{\rho}, \boldsymbol{\epsilon})$ & GNN prediction under restoration configuration
        $\boldsymbol{\rho}$ and NC realization
        $\boldsymbol{\epsilon}$. \\

        \addlinespace[1pt]
        \midrule
        \multicolumn{2}{@{}l}{
        \textbf{Stochastic Restoration Boundary Learning}}\\[-1pt]
        $[x]_+$ & Positive-part operator defined as $\max(x,0)$. \\

        $d_{y^\star}(\boldsymbol{\rho},\boldsymbol{\epsilon})$ & Target-prediction degradation under NC. \\

        $\mathcal{R}_{y^\star}(\boldsymbol{\rho})$ & Stochastic restoration risk of $\boldsymbol{\rho}$. \\

        $\beta$ & Coefficient of the variation penalty. \\

        $\lambda_{\mathrm{rest}}$ & Coefficient of the compactness penalty. \\

        $\boldsymbol{\rho}^{\star}$ & Learned stochastic restoration boundary. \\

        $\rho_e^{\star}$ & Restoration degree at the learned boundary. \\

        $N$ & Number of NC samples. \\

        $\boldsymbol{\epsilon}^{(n)}$ & The $n$-th independently sampled noise. \\

        $d_n$ &  The $n$-th target-prediction degradation. \\

        $\bar d_N$ & Mean target-prediction degradation. \\

        $\widehat{\mathcal{R}}_{y^\star}^{(N)}
        (\boldsymbol{\rho})$ & Estimated empirical restoration risk. \\

        $\|\boldsymbol{\rho}\|_1$ & Regularization boundary compactness. \\

        \addlinespace[1pt]
        \midrule
        \multicolumn{2}{@{}l}{\textbf{Boundary-Integrated Gradient}}\\[-1pt]

        $t\in[0,1]$ &  The continuous restoration path. \\

        $\boldsymbol{\rho}(t)$ & Restoration configuration at position $t$. \\

        $T$ & Number of numerical integration steps. \\

        $s_e^{\mathrm{BIG}}$ & BIG attribution of edge $e$. \\

        \addlinespace[1pt]
        \midrule
        \multicolumn{2}{@{}l}{\textbf{Operators and diagnostic metrics}}\\[-1pt]

        $\|\cdot\|_2$
        & Euclidean norm. \\

        $\|\cdot\|_1$
        & $\ell_1$ norm. \\

        $\mathbb{E}_{\boldsymbol{\epsilon}}[\cdot]$
        & Expectation over NC realizations. \\

        $\operatorname{Std}_{\boldsymbol{\epsilon}}[\cdot]$
        & Standard deviation over NC realizations. \\

        $D_{\mathrm{repr}}$
        &  Graph representations distance. \\

        $D_{\mathrm{pred}}$
        & Degradation of the target prediction. \\

        $\mathrm{Rest.}$
        & Average restoration  boundary. \\

        \bottomrule
    \end{tabular}
    }
    \caption{Summary of notation (Part 2).}
    \label{tab:notation_2}
\end{table}

\begin{table*}[htbp]
\centering
\resizebox{0.8\linewidth}{!}{% 适配单栏宽度，双栏模板可改为 \columnwidth
% \small
\begin{tabular}{c|cccccc}
\toprule
Statistic & Graphs  & Average Nodes & Average Edges & Node Features & Classes & Train/Test \\
\midrule
Mutag & 1,768 &  \textasciitilde 29.15 &  \textasciitilde 60.83 & 14 & 2 & 1597/177 \\
Benzene & 12,000 & \textasciitilde 20.58 & \textasciitilde 43.64 & 14 & 2 & 10800/1200 \\
Alkane & 1,125 & \textasciitilde 21.39 &  \textasciitilde 45.38 & 14 & 2 & 1012/113 \\
Fluoride & 8,671 & \textasciitilde 21.36 & \textasciitilde 45.37 & 14 & 2 & 7803/868 \\
B-XAIC & 5,000 & \textasciitilde 34.60 & \textasciitilde 75.19 & 11 & 2 & 45000/5000\\
% X &5,000 & Benzene & Alkane + C=O & $\ce{F^-}$ + C=O & Motif & \\
% P & 5,000 & Benzene & Alkane + C=O & $\ce{F^-}$ + C=O & Motif & \\
% Indole & 5,000 & Benzene & Alkane + C=O & $\ce{F^-}$ + C=O & Motif & \\
% PAINS & 5,000 & Benzene & Alkane + C=O & $\ce{F^-}$ + C=O & Motif & \\
% Rings count & 5,000 & Benzene & Alkane + C=O & $\ce{F^-}$ + C=O & Motif & \\
% Rings max & 5,000 & Benzene & Alkane + C=O & $\ce{F^-}$ + C=O & Motif & \\
\bottomrule
\end{tabular}
}
\caption{The statistics of the evaluated datasets and the split.}
\label{tab:dataset_stats}
\end{table*}

\subsection{Proof of Theorem~\ref{theorem_em}}

Recall that the layer-wise message scale is defined as
\begin{equation}
\mathcal S^{(l)}(G)
=
\frac{1}{|\mathcal E|}
\sum_{a\in\mathcal E}
\left\|
\mathbf m_a^{(l)}
\right\|_2^2.
\label{eq:proof_clean_scale}
\end{equation}
To isolate the direct effect of EM at layer $l$, we hold the
incoming layer representations fixed and apply the mask
coefficients $\{\rho_a\}_{a\in\mathcal E}$ to the resulting
messages. By Eq.~\eqref{eq:em}, the masked message induced by
edge $a$ is
\begin{equation}
\widetilde{\mathbf m}_a^{(l)}
=
\rho_a\mathbf m_a^{(l)}.
\end{equation}
Therefore, the corresponding message scale is
\begin{align}
\mathcal S_{\mathrm{EM}}^{(l)}(G)
&=
\frac{1}{|\mathcal E|}
\sum_{a\in\mathcal E}
\left\|
\widetilde{\mathbf m}_a^{(l)}
\right\|_2^2
\nonumber\\
&=
\frac{1}{|\mathcal E|}
\sum_{a\in\mathcal E}
\rho_a^2
\left\|
\mathbf m_a^{(l)}
\right\|_2^2.
\label{eq:proof_masked_scale}
\end{align}
Dividing Eq.~\eqref{eq:proof_masked_scale} by
Eq.~\eqref{eq:proof_clean_scale} gives
\begin{align}
\frac{
\mathcal S_{\mathrm{EM}}^{(l)}(G)
}{
\mathcal S^{(l)}(G)
}
&=
\frac{
\sum_{a\in\mathcal E}
\rho_a^2
\|\mathbf m_a^{(l)}\|_2^2
}{
\sum_{a\in\mathcal E}
\|\mathbf m_a^{(l)}\|_2^2
}
\nonumber\\
&=
\sum_{a\in\mathcal E}
\rho_a^2\eta_a^{(l)},
\label{eq:proof_weighted_mask}
\end{align}
where
\begin{equation}
\eta_a^{(l)}
=
\frac{
\|\mathbf m_a^{(l)}\|_2^2
}{
\sum_{b\in\mathcal E}
\|\mathbf m_b^{(l)}\|_2^2
}.
\end{equation}
Since
$\sum_{a\in\mathcal E}\eta_a^{(l)}=1$, we have
\begin{align}
\sum_{a\in\mathcal E}
\rho_a^2\eta_a^{(l)}
&=
\sum_{a\in\mathcal E}
\left[
1-(1-\rho_a^2)
\right]
\eta_a^{(l)}
\nonumber\\
&=
1-
\sum_{a\in\mathcal E}
(1-\rho_a^2)\eta_a^{(l)}.
\label{eq:proof_contraction}
\end{align}
Because $\rho_a\in[0,1]$ and
$\eta_a^{(l)}\geq0$, every term
$(1-\rho_a^2)\eta_a^{(l)}$ is non-negative. Hence,
\begin{equation}
\frac{
\mathcal S_{\mathrm{EM}}^{(l)}(G)
}{
\mathcal S^{(l)}(G)
}
=
1-
\sum_{a\in\mathcal E}
(1-\rho_a^2)\eta_a^{(l)}
\leq 1.
\end{equation}
Moreover, the inequality is strict whenever there exists an
edge $a$ such that
$\eta_a^{(l)}>0$ and $\rho_a<1$, because the corresponding
term
$(1-\rho_a^2)\eta_a^{(l)}$
is then strictly positive.
This completes the proof.
\hfill$\square$

\subsection{Proof of Corollary~\ref{corollary_depth}}

Suppose that an edge $e$ is masked with a fixed coefficient
$\rho_e\in(0,1)$ and satisfies
\begin{equation}
\eta_e^{(l)}
\geq
\underline{\eta}
>
0,
\qquad
l=1,\ldots,L.
\end{equation}
Theorem~\ref{theorem_em} gives, for every layer $l$,
\begin{align}
\frac{
\mathcal S_{\mathrm{EM}}^{(l)}(G)
}{
\mathcal S^{(l)}(G)
}
&=
1-
\sum_{a\in\mathcal E}
(1-\rho_a^2)\eta_a^{(l)}
\nonumber\\
&\leq
1-
(1-\rho_e^2)\eta_e^{(l)}
\nonumber\\
&\leq
1-
(1-\rho_e^2)\underline{\eta}.
\label{eq:proof_single_edge_bound}
\end{align}
The first inequality follows because all omitted terms are
non-negative.
Define
\begin{equation}
q
=
1-
(1-\rho_e^2)\underline{\eta}.
\end{equation}
Since $\rho_e\in(0,1)$ and
$\underline{\eta}>0$, we have
$0<q<1$.
Because the function $-\log x$ is monotonically decreasing
on $(0,\infty)$, Eq.~\eqref{eq:proof_single_edge_bound}
implies
\begin{equation}
-\log
\frac{
\mathcal S_{\mathrm{EM}}^{(l)}(G)
}{
\mathcal S^{(l)}(G)
}
\geq
-\log q.
\end{equation}
Summing the above inequality over
$l=1,\ldots,L$ yields
\begin{align}
\mathcal D_{\mathrm{EM}}^{(L)}
&=
-\sum_{l=1}^{L}
\log
\frac{
\mathcal S_{\mathrm{EM}}^{(l)}(G)
}{
\mathcal S^{(l)}(G)
}
\nonumber\\
&\geq
-L\log q
\nonumber\\
&=
-L
\log
\left[
1-(1-\rho_e^2)\underline{\eta}
\right].
\end{align}
The right-hand side is linear in $L$ with a strictly positive
coefficient because $q\in(0,1)$.
Therefore, the cumulative log-scale drift grows at least
linearly with the number of message-passing layers.
\hfill$\square$
\subsection{Proof of Proposition~\ref{prop:nc_scale_stability}}

For brevity, let
$\mathbf m=\mathbf m_e^{(l)}$,
$\boldsymbol\epsilon=\boldsymbol\epsilon_e^{(l)}$,
and $\rho=\rho_e$.
By the definition of NC,
\begin{equation}
\Psi_{\mathrm{NC},\rho}^{(l)}(\mathbf m)
=
\sqrt{\rho}\,\mathbf m
+
\sqrt{1-\rho}\,\boldsymbol\epsilon.
\end{equation}
Expanding its squared norm gives
\begin{align}
\left\|
\Psi_{\mathrm{NC},\rho}^{(l)}(\mathbf m)
\right\|_2^2
&=
\rho\|\mathbf m\|_2^2
+
(1-\rho)\|\boldsymbol\epsilon\|_2^2
\nonumber\\
&\quad+
2\sqrt{\rho(1-\rho)}
\mathbf m^\top\boldsymbol\epsilon.
\label{eq:nc_norm_expansion}
\end{align}
The corrupted counterpart is sampled uniformly from the
sphere
\begin{equation}
\left\{
\mathbf z:
\|\mathbf z\|_2
=
\|\mathbf m\|_2
\right\}.
\end{equation}
Therefore,
\begin{equation}
\|\boldsymbol\epsilon\|_2^2
=
\|\mathbf m\|_2^2
\end{equation}
almost surely.
Moreover, the uniform distribution on a sphere is symmetric
around the origin, and hence
\begin{equation}
\mathbb E_{\boldsymbol\epsilon}
[\boldsymbol\epsilon]
=
\mathbf 0.
\end{equation}
Consequently,
\begin{equation}
\mathbb E_{\boldsymbol\epsilon}
[
\mathbf m^\top\boldsymbol\epsilon
]
=
\mathbf m^\top
\mathbb E_{\boldsymbol\epsilon}
[\boldsymbol\epsilon]
=
0.
\end{equation}
Taking expectation on both sides of
Eq.~\eqref{eq:nc_norm_expansion}, we obtain
\begin{align}
\mathbb E_{\boldsymbol\epsilon}
\left[
\left\|
\Psi_{\mathrm{NC},\rho}^{(l)}(\mathbf m)
\right\|_2^2
\right]
&=
\rho\|\mathbf m\|_2^2
+
(1-\rho)\|\mathbf m\|_2^2
\nonumber\\
&=
\|\mathbf m\|_2^2.
\end{align}
Thus, NC preserves the squared message norm in expectation
for every restoration gate $\rho\in[0,1]$.
\hfill$\square$

\begin{table*}[!t]
\centering
\resizebox{0.85\textwidth}{!}{
\begin{tabular}{ccccc}
\toprule
\textbf{Dataset}
& \textbf{Target Acc. (\%)}
& \makecell{\textbf{Positive}\\\textbf{Ratio (\%)}}
& \textbf{Ground-Truth Explanation}
& \makecell{\textbf{GT Edge Ratio (\%)}\\
            \textbf{Mean $\pm$ Std.}} \\
\midrule
Mutagencity & 100.00 & 36.20 & $\ce{NO2}$ and $\ce{NH2}$ groups & $3.78 \pm 6.85$ \\
Benzene & 93.17 & 50.00 & Benzene ring & $13.17 \pm 13.60$ \\
Alkane & 100.00 & 33.33 & Alkane chain and $\ce{C=O}$ group & $1.46 \pm 2.10$ \\
Fluoride & 96.31 & 17.61 & Fluorine atom and $\ce{C=O}$ group & $2.69 \pm 6.04$ \\

\midrule
Indole & 98.68 & 36.70 & Fused benzene--pyrrole structure & $11.54 \pm 16.85$ \\
PAINS & 91.84 & 32.94 & Detected PAINS alert substructure & $10.34 \pm 17.01$ \\
Rings-Count & 94.10 & 30.20 & Annotated ring structures
& $61.00 \pm 16.38$ \\

Rings-Max
& 96.16
& 5.65
& Ring containing more than six atoms
& $46.48 \pm 17.15$ \\
\bottomrule
\end{tabular}
}
\caption{
Predictive performance of the target GINs and statistics of
the ground-truth explanations.
The GT edge ratio is the proportion of ground-truth
explanatory edges among all edges, averaged over the graphs
included in explanation evaluation.
}
\label{tab:model_stats}
\end{table*}

\section{C. Datasets and Target GNNs}\label{app:datasets}

We evaluate NICE on eight molecular graph-classification benchmarks with ground-truth explanatory structures.
The first four are widely used molecular explanation benchmarks, while the remaining four are selected tasks from
the B-XAIC benchmark. 
In all datasets, atoms are represented as graph nodes and chemical bonds as graph edges. 
Since NICE produces edge-level attributions, we use the annotated bonds belonging to the corresponding chemical patterns as ground-truth explanatory edges.
The dataset details are introduced as follows and the dataset statistics are summarized in Table~\ref{tab:dataset_stats}.

For each dataset, we train an independent GIN classifier as the target model to be explained. 
Each model uses a hidden dimension of $32$ and is optimized with Adam for at most $1{,}000$ epochs using a learning rate of $0.01$ and a batch size of $2{,}048$. 
We employ cross-entropy loss and reduce the learning rate by a factor of $0.5$ when the validation loss does not improve for $100$ consecutive epochs. 
The checkpoint with the highest validation accuracy is selected, after which its performance is evaluated once on the held-out test set.

Table~\ref{tab:model_stats} reports the predictive accuracy of the resulting target GINs together with the positive-label ratio and ground-truth explanation statistics. 
For each evaluated graph $G$, we define the ground-truth edge ratio as
\begin{equation}
r_{\mathrm{GT}}(G)
=
\frac{|\mathcal E_{\mathrm{GT}}(G)|}{|\mathcal E(G)|},
\end{equation}
where $\mathcal E_{\mathrm{GT}}(G)$ denotes its
ground-truth explanatory edges.
We report the mean and standard deviation of
$r_{\mathrm{GT}}(G)$ over the graphs included in explanation
evaluation.
The target GINs attain high predictive accuracy across the
eight tasks, ensuring that the subsequent experiments evaluate
explanation quality rather than classifier failure.

\subsection{Molecular Explanation Benchmarks}

\paragraph{Mutagenicity.}
Mutagenicity~\cite{mutag} contains $4{,}337$ molecular graphs labeled according to whether the corresponding compound exhibits mutagenic activity. 
The annotated nitro and amino functional groups ($\mathrm{NO_2}$ and $\mathrm{NH_2}$) serve as the ground-truth explanations. 
The dataset contains on average $30.32$ nodes and $30.77$ edges per graph.

\paragraph{Benzene.}
Benzene~\cite{graphxai} is a binary classification dataset containing $12{,}000$ molecular graphs
sampled from ZINC15. 
The task is to determine whether a molecule contains at least one benzene ring. 
The atoms and bonds forming each detected benzene ring are annotated as the ground-truth explanation. 
When a molecule contains multiple disjoint benzene rings, each ring is regarded as a valid explanatory structure. 
The graphs contain on average $20.58$ nodes and $43.65$ edges.

\paragraph{Alkane-Carbonyl.}
Alkane-Carbonyl~\cite{graphxai} contains $4{,}326$ molecular graphs. 
A molecule receives a positive label when it simultaneously contains an unbranched alkane chain and a carbonyl ($\mathrm{C{=}O}$) group. 
The ground-truth explanation is defined as the union of the annotated alkane chain and carbonyl group. 
The average graph contains $21.13$ nodes and $44.95$ edges.

\paragraph{Fluoride-Carbonyl.}
Fluoride-Carbonyl~\cite{graphxai} consists of $8{,}671$ molecular graphs. 
A positive molecule must contain both a fluorine atom and a carbonyl ($\mathrm{C{=}O}$) group. 
The fluorine atom, the carbonyl group, and their corresponding annotated bonds jointly form the ground-truth explanation. 
The graphs contain on average $21.36$ nodes and $45.37$ edges.

\subsection{B-XAIC Tasks}

B-XAIC~\cite{B-XAIC} is constructed from
ChEMBL~35, which contains approximately $2.5$ million
drug-like molecules.
Invalid and duplicated SMILES strings are removed, and
solvents and counterions are discarded to retain one molecular
graph per example.
Weighted sampling is then used to obtain $50{,}000$ molecules,
which are divided into training, validation, and test sets of
$40{,}000$, $5{,}000$, and $5{,}000$ graphs, respectively.
The resulting molecular graphs contain $34.56$ atoms on
average.
Each molecule is associated with binary task labels and with
atom- and bond-level annotations for the detected chemical
patterns.
The four tasks below therefore share the same molecular
collection but differ in their prediction targets and
ground-truth explanations.

\paragraph{Indole.}
The Indole task predicts whether a molecule contains an indole
group, i.e., a bicyclic structure consisting of a benzene ring
fused with a pyrrole ring.
Detecting this relatively large pattern requires information to
be propagated across multiple atoms.
For positive molecules, the atoms and bonds forming the indole
structure constitute the ground-truth explanation.
Approximately $36.94\%$ of the B-XAIC molecules receive a
positive label for this task.

\paragraph{PAINS.}
The PAINS task detects pan-assay interference compounds,
whose characteristic substructures are known to produce
false-positive outcomes in high-throughput screening.
Unlike tasks defined by a single motif, PAINS involves a
diverse collection of chemically distinct alert patterns.
A molecule is labeled positive when it contains at least one
of these patterns, and the atoms and bonds belonging to the
detected PAINS substructure form the ground-truth explanation.
The positive-label ratio is approximately $32.88\%$.

\paragraph{Rings-Count.}
The Rings-Count task predicts whether a molecule contains more
than four rings.
It therefore requires not only detecting ring structures but
also counting their occurrences within the molecular graph.
The atoms and bonds participating in the ring structures
relevant to the count are provided as ground-truth explanation
annotations.
Approximately $30.06\%$ of the molecules are positive for
this task.

\paragraph{Rings-Max.}
The Rings-Max task predicts whether a molecule contains a ring with more than six atoms. 
In contrast to Rings-Count, which concerns the number of rings, this task requires determining the size of individual ring structures. 
For positive instances, the atoms and bonds forming the qualifying large ring constitute the ground-truth explanation. 
This is the most imbalanced of the four selected B-XAIC tasks, with a positive-label ratio of approximately $5.54\%$.

\section{D. Related Work}
\label{sec:related_work}

\paragraph{Post-hoc GNN Explanation} 
Post-hoc GNN explainers aim to identify the nodes, edges, or subgraphs responsible for a prediction of a trained GNN. 
Gradient-based methods directly measure prediction sensitivity with respect to graph features or structures \cite{guided-bp, SA}, whereas search- and decomposition-based methods construct explanatory subgraphs or decompose predictions into structural contributions \cite{subgraphX,graphmask}. 
A prominent line follows the Perturb-Query paradigm, in which graph elements are perturbed and their importance is inferred from the response of the target GNN. 
GNNExplainer learns an instance-specific soft mask by maximizing mutual information \cite{ying2019gnnexplainer}, while PGExplainer amortizes mask generation across graph instances
\cite{pge}. 
ReFine further introduces class-aware explanation generation to capture contrastive structures \cite{ReFine}. 
Despite their different score-generation and optimization strategies, many Perturb-Query explainers instantiate continuous scores through multiplicative suppression of edges
or their induced messages. 
NICE belongs to this general paradigm but revisits the message-level perturbation mechanism used to query the target GNN.

\paragraph{Distribution Shift in Perturb-Query Explanation}
Perturbing a graph may move the queried input or its internal representation away from the operating regime of the target GNN, making prediction changes difficult to interpret \cite{OAR,robust_f}. 
Existing methods address this issue at different stages of the explanation pipeline. 
D4Explainer formulates explanation generation as a discrete denoising process to improve the in-distribution property of generated graphs \cite{d4e}. 
Mixup-based methods combine explanatory structures with base graphs to reduce the discrepancy between explanation subgraphs and the original data distribution \cite{mixup}. 
V-InFoR learns robust graph representations under structural corruption \cite{vinfor}, while ProxyExplainer constructs in-distribution proxy graphs for querying the target model \cite{proxy}. 
More recently, IDEA aligns input graphs and explanations in a prototypical representation space \cite{yin2026paradig}, and ConfExplainer corrects unreliable query signals through confidence-aware optimization \cite{Zhang2025IsYE}. 
Related studies also revisit how faithfulness should be evaluated when graph perturbations themselves introduce distribution shift \cite{robust-counter}.

These approaches primarily improve the generated explanations, their representations, or the optimization and evaluation of queried predictions. 
This work examines a complementary source of ambiguity: multiplicative masking simultaneously suppresses edge-specific message signals and contracts propagation scale. 
We therefore do not regard NC as a smoother approximation to edge deletion. 
Instead, NICE defines a scale-controlled restoration intervention that measures how restoring clean-message directions contributes to the original target prediction under matched expected message scale. 

\paragraph{Path-Based Attribution} 
Path-based attribution methods accumulate gradients along a continuous path from a reference state to the input, with
Integrated Gradients being a representative approach \cite{sundararajan2017axiomatic}. 
In graph explanation, gradients with respect to edge weights have also been connected theoretically to perturbation-based and occlusion-based explanations under specific model assumptions \cite{automatic_concept}. 
BIG adopts the path-integration principle but differs in both its path and attribution objective. 
It integrates gradients of stochastic restoration risk in the restoration-gate space, from the fully direction-corrupted
state to the boundary learned by SRB. 
Consequently, BIG attributes each edge's contribution to target-prediction restoration rather than its input sensitivity or deletion effect.

\section{E. Implement Details}

\label{app:implementation}

This section provides the optimization procedure of
Stochastic Restoration Boundary Learning (SRB), the
explanation-generation procedure based on
Boundary-Integrated Gradient (BIG), and the configurations
used in our experiments.

\begin{algorithm}[t]
\caption{BIG-Based Explanation Generation}
\label{alg:big}
\begin{algorithmic}[1]
\REQUIRE Graph $G$, frozen target GNN $f$, trained explainer
$g_{\boldsymbol\theta^\star}$, NC sample number $M$, and
integration steps $T$
\ENSURE Edge-level attribution
$\mathbf s^{\mathrm{BIG}}$

\STATE Perform a clean forward pass and obtain
$y^*=\arg\max_c f_c(G)$
\STATE Compute edge representations
$\{\mathbf z_e\}_{e\in\mathcal E}$
\STATE Compute the restoration boundary
$\rho_e^\star
=\sigma(g_{\boldsymbol\theta^\star}(\mathbf z_e))$

\FOR{$k=1,\ldots,T$}
    \STATE Set
    $\boldsymbol\rho^{(k)}
    =\frac{k}{T}\boldsymbol\rho^\star$
    \STATE Estimate
    $\widehat{\mathcal R}_{y^*}
    (\boldsymbol\rho^{(k)})$
    using $M$ independent NC samples
    \STATE Compute
    $\nabla_{\boldsymbol\rho}
    \widehat{\mathcal R}_{y^*}
    (\boldsymbol\rho^{(k)})$
\ENDFOR

\STATE Compute
\[
\mathbf s^{\mathrm{BIG}}
=
-
\boldsymbol\rho^\star
\odot
\frac{1}{T}
\sum_{k=1}^{T}
\nabla_{\boldsymbol\rho}
\widehat{\mathcal R}_{y^*}
(\boldsymbol\rho^{(k)})
\]
\STATE Map message-level scores to graph edges and rank them
in descending order
\RETURN $\mathbf s^{\mathrm{BIG}}$
\end{algorithmic}
\end{algorithm}

\subsection{SRB Training Stage}
\label{app:srb_training}

\paragraph{Explainer parameterization.}
The target GNN is pretrained and remains frozen throughout explainer optimization. 
For each edge $e$, we obtain an edge representation $\mathbf z_e$ from the hidden representations produced by the target GNN and parameterize its restoration gate as
\begin{equation}
\rho_e
=
\sigma\left(g_{\boldsymbol\theta}(\mathbf z_e)\right),
\end{equation}
where $g_{\boldsymbol\theta}$ is a multilayer perceptron and
$\sigma(\cdot)$ is the sigmoid function.
The same edge-level restoration gate is used across all message-passing layers.
Only the parameters $\boldsymbol\theta$ of the explainer are updated during SRB training.

\paragraph{Stochastic restoration risk estimation.}
For each graph $G$, we first perform a clean forward pass to obtain its original target prediction
\begin{equation}
y^*=\arg\max_c f_c(G)
\end{equation}
and the corresponding target-class probability
$p_{y^*}(G)$.
Given the restoration configuration
$\boldsymbol\rho=\{\rho_e\}_{e\in\mathcal E}$, we sample
$N$ independent collections of matched-norm corrupted messages:
\begin{equation}
\boldsymbol\epsilon^{(n)}
\sim\mathcal P_{\mathrm{NC}},
\qquad n=1,\ldots,N.
\end{equation}
Each collection defines one independent NC forward pass. 
For the $n$-th NC sample, the one-sided target-prediction degradation is
\begin{equation}
d_n
=
\left[
\log p_{y^*}(G)
-
\log p_{y^*}
\left(
\boldsymbol\rho,
\boldsymbol\epsilon^{(n)}
\right)
\right]_+.
\end{equation}
The empirical stochastic restoration risk is computed as
\begin{equation}
\widehat{\mathcal R}_{y^*}^{(N)}
(\boldsymbol\rho)
=
\overline d_N
+
\beta
\sqrt{
\frac{1}{N}
\sum_{n=1}^{N}
\left(d_n-\overline d_N\right)^2
},
\end{equation}
where
\begin{equation}
\overline d_N
=
\frac{1}{N}
\sum_{n=1}^{N}d_n.
\end{equation}

The NC samples are resampled at every optimization step. 
Consequently, the explainer is optimized over the distribution of matched-norm corruptions rather than a fixed set of noise samples. 

\paragraph{Batch-level optimization.}
For a mini-batch $\mathcal B$, we optimize
\begin{equation}
\mathcal L_{\mathcal B}
=
\frac{1}{|\mathcal B|}
\sum_{G\in\mathcal B}
\left[
\widehat{\mathcal R}_{y^*}^{(N)}
(\boldsymbol\rho_G)
+
\lambda_{\mathrm{rest}}
\frac{
\|\boldsymbol\rho_G\|_1
}{
|\mathcal E_G|
}
\right].
\label{eq:batch_srb}
\end{equation}
The first term encourages accurate and stable target-prediction restoration across NC samples, whereas the second term penalizes unnecessary restoration. 
Normalizing the compactness term by the number of edges prevents graph size from changing its effective strength.

The explainer parameters are optimized using Adam. 
After training, the learned restoration configuration for a graph is
\begin{equation}
\boldsymbol\rho^\star_G
=
\left\{
\sigma
\left(
g_{\boldsymbol\theta^\star}(\mathbf z_e)
\right)
\right\}_{e\in\mathcal E_G}.
\end{equation}
Each $\rho_e^\star$ represents the restoration degree of edge $e$ at the learned stochastic restoration boundary; it is not directly used as the final edge attribution.

\subsection{BIG-Based Explanation Generation}
\label{app:big_generation}

Once SRB training is complete, the explainer parameters and the target GNN are both fixed.
For each graph, we first compute its restoration boundary $\boldsymbol\rho^\star$ and then use BIG to convert the joint restoration process into edge-level attributions.

\paragraph{Restoration path.}
We construct a linear path in the restoration-gate space:
\begin{equation}
\boldsymbol\rho(t)
=
t\boldsymbol\rho^\star,
\qquad t\in[0,1].
\end{equation}
The starting point $\boldsymbol\rho(0)=\mathbf 0$ corresponds to fully direction-corrupted messages, whereas $\boldsymbol\rho(1)=\boldsymbol\rho^\star$ reaches the
learned restoration boundary. 
The path is linear in the gate space; because NC uses $\sqrt{\rho_e}$ in the message transformation, it is not a linear interpolation in the message-vector space.

\paragraph{Numerical integration.}
We approximate the BIG integral using a right-endpoint Riemann sum with $T$ uniformly spaced path points:
\begin{equation}
t_k=\frac{k}{T},
\qquad
\boldsymbol\rho^{(k)}
=
t_k\boldsymbol\rho^\star,
\qquad
k=1,\ldots,T.
\end{equation}
At every path point, we estimate the gradient of stochastic restoration risk with respect to the restoration gates. 
The resulting attribution is
\begin{equation}
\widehat s_e^{\mathrm{BIG}}
=
-
\frac{\rho_e^\star}{T}
\sum_{k=1}^{T}
\left.
\frac{
\partial
\widehat{\mathcal R}_{y^*}
(\boldsymbol\rho)
}{
\partial\rho_e
}
\right|_{
\boldsymbol\rho=
\boldsymbol\rho^{(k)}
}.
\label{eq:big_approximation}
\end{equation}
The factor $\rho_e^\star$ accounts for the total displacement of edge $e$ from the fully corrupted state to the learned boundary. 
The negative sign assigns a positive contribution when restoring the edge reduces stochastic restoration risk.

\paragraph{Edge ranking.}
The BIG scores are computed for all candidate edges and ranked in descending order to obtain the final explanation. 
For undirected molecular graphs, the two directed message-passing entries corresponding to the same chemical bond share one restoration gate and are mapped back to one bond-level attribution. 
The resulting scores are used directly for ranking-based metrics, while the Top-$k$ edges are retained for threshold-dependent metrics such as F1, Recall, Precision and Fidelity.

% Recall & Precision on various Tasks
\begin{table*}[htbp]
  \centering
\resizebox{\linewidth}{!}{
\renewcommand{\arraystretch}{1.3}
\begin{tabular}{cc|cccccccc}
\toprule
\textbf{Metric} & \textbf{Method} & \textbf{Mutag} & \textbf{Benzene} & \textbf{Alkane} & \textbf{Fluoride} & \textbf{Indole} & \textbf{\textsc{PAINS}} & \textbf{R-Count} & \textbf{R-Max}\\

\multirow{10}{*}{\textbf{Prec}}
& Random
& \mstd{15.58}{0.68}
& \mstd{29.90}{0.57}
& \mstd{5.17}{0.63}
& \mstd{19.23}{0.34}
& \mstd{31.46}{0.08}
& \mstd{31.74}{0.19}
& \mstd{68.52}{0.15}
& \mstd{30.96}{1.03}\\

& Saliency
& \mstd{38.70}{0.24}
& \mstd{44.98}{0.02}
& \cellcolor{second}\underline{\mstd{0.00}{0.00}}
& \mstd{37.79}{0.00}
& \mstd{51.52}{0.00}
& \mstd{52.39}{0.02}
& \mstd{63.18}{0.01}
& \mstd{41.14}{0.00}\\

& GuidedBP
& \mstd{30.74}{0.00}
& \mstd{43.35}{0.02}
& \mstd{8.86}{0.00}
& \mstd{33.73}{0.00}
& \mstd{44.85}{0.00}
& \cellcolor{second}\underline{\mstd{54.74}{0.01}}
& \mstd{64.43}{0.01}
& \mstd{41.67}{0.00}\\

& GNNExplainer
& \mstd{24.47}{1.51}
& \mstd{19.68}{0.44}
& \cellcolor{best}\textbf{\mstd{24.84}{1.87}}
& \mstd{23.20}{0.68}
& \mstd{43.41}{0.55}
& \mstd{42.31}{0.10}
& \mstd{85.03}{1.08}
& \mstd{42.59}{1.33}\\

& PGExplainer
& \cellcolor{second}\underline{\mstd{39.74}{0.22}}
& \mstd{66.20}{0.13}
& \mstd{15.53}{0.00}
& \cellcolor{second}\underline{\mstd{49.96}{0.33}}
& \mstd{48.76}{0.03}
& \mstd{49.00}{0.00}
& \mstd{88.02}{0.00}
& \cellcolor{best}\textbf{\mstd{51.15}{0.00}}\\

& Refine
& \mstd{34.43}{0.75}
& \mstd{42.77}{0.47}
& \mstd{6.17}{0.00}
& \mstd{19.84}{2.11}
& \mstd{40.89}{0.08}
& \cellcolor{best}\textbf{\mstd{58.01}{0.05}}
& \mstd{86.40}{0.08}
& \mstd{48.65}{0.00}\\

& D4Explainer
& \mstd{23.29}{0.76}
& \mstd{47.48}{0.20}
& \cellcolor{second}\underline{\mstd{15.73}{0.54}}
& \mstd{21.54}{0.91}
& \mstd{39.11}{0.69}
& \mstd{43.86}{0.21}
& \cellcolor{best}\textbf{\mstd{92.43}{2.13}}
& \mstd{44.40}{1.60}\\

& ProxyExplainer
& \mstd{30.79}{0.00}
& \mstd{48.54}{0.04}
& \mstd{8.45}{0.00}
& \mstd{29.77}{0.00}
& \mstd{40.91}{0.19}
& \mstd{38.62}{0.03}
& \cellcolor{second}\underline{\mstd{89.40}{0.00}}
& \mstd{48.04}{0.02}\\

& ConfExplainer
& \cellcolor{best}\textbf{\mstd{42.57}{0.00}}
& \cellcolor{second}\underline{\mstd{73.46}{0.01}}
& \mstd{14.56}{0.00}
& \mstd{40.81}{0.06}
& \cellcolor{second}\underline{\mstd{60.38}{0.00}}
& \cellcolor{second}\underline{\mstd{59.08}{0.01}}
& \mstd{86.42}{0.00}
& \cellcolor{second}\underline{\mstd{48.99}{0.00}}\\

& \textbf{NICE(Ours)}
& \cellcolor{second}\underline{\mstd{41.70}{0.19}}
& \cellcolor{best}\textbf{\mstd{77.10}{0.05}}
& \cellcolor{best}\textbf{\mstd{16.10}{0.07}}
& \cellcolor{best}\textbf{\mstd{53.64}{0.04}}
& \cellcolor{best}\textbf{\mstd{72.29}{0.00}}
& \mstd{57.04}{0.00}
& \mstd{90.61}{0.00}
& \mstd{50.90}{0.00}\\

\bottomrule
\end{tabular}
}
\caption{
Precision of NICE and the compared baselines. 
The best result is highlighted in bold, and the second-best result is underlined.
}
\label{tab:precision}
\end{table*}

\subsection{Experimental Configuration}
\label{app:configuration}

\paragraph{Target GNNs.}
We train an independent GIN classifier for each dataset. 
The target GNN uses a hidden dimension of $32$ and is optimized with Adam using cross-entropy loss. 
The initial learning rate is $0.01$, the batch size is $2{,}048$, and the maximum number of training epochs is $1{,}000$. 
A ReduceLROnPlateau scheduler reduces the learning rate by a factor of $0.5$ when the validation loss does not improve for $100$ consecutive epochs. 
The checkpoint with the highest validation accuracy is used as the target model and remains frozen during explanation. 

\paragraph{NICE configuration.}
The restoration-gate predictor is optimized using Adam with a learning rate of $0.001$. 
Unless otherwise stated, we use $N=16$ independent NC samples to estimate the stochastic restoration risk, set the variation coefficient to $\beta=0.1$, and set the compactness coefficient to $\lambda_{\mathrm{rest}}=1.0$. 
BIG is approximated using $T=16$ uniformly spaced integration steps. 
All reported explanation results are averaged over five random seeds. 
All experiments are finished on a machine with 4 \textit{NVIDIA GEFORCE RTX 3090 24GIB} GPUs.

\begin{figure}[!htbp]
    \centering
    \includegraphics[width=1\columnwidth]{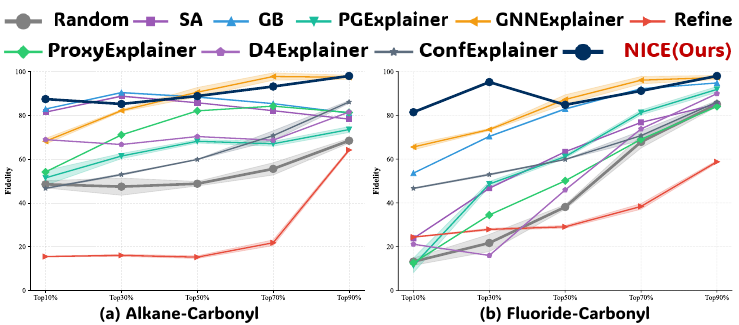}
\caption{
Fidelity under varying edge-retention ratios on
Alkane-Carbonyl and Fluoride-Carbonyl.
NICE preserves the target prediction particularly well under
compact explanation budgets, while the performance gap
gradually narrows as more edges are retained.
Higher values indicate better model faithfulness.
}
\label{fig:additional_fidelity}
\end{figure}

\section{F. Additional Experimental Results}

\paragraph{Explanation Performance (Precision).}
Table~\ref{tab:precision} reports the Precision results. NICE achieves the highest macro-average Precision of $57.42\%$, outperforming the strongest baseline, ConfExplainer, by $4.14$ percentage points. 
It ranks first on Benzene, Fluoride, and Indole, and second on Mutag, Alkane, Rings-Count, and Rings-Max, achieving a top-two result on seven of the eight datasets. 
The improvement is particularly pronounced on Indole, where NICE increases Precision from $60.38\%$ to $72.29\%$. 
Although NICE is slightly behind ConfExplainer and ReFine on PAINS, it remains competitive.
Together with the Recall and F1 results in Table~\ref{tab:main_results}, these results indicate that the improved coverage of ground-truth explanatory edges does not come at the cost of substantially reduced explanatory selectivity.

\begin{figure*}[!htbp]
    \centering
    \includegraphics[width=1\linewidth]{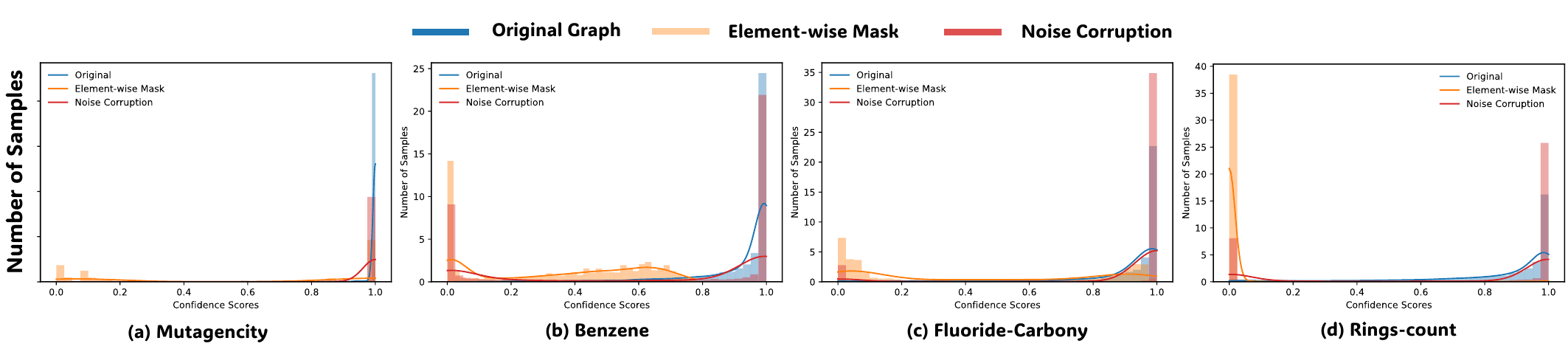}
\caption{
Distribution of target-class prediction confidence under the original graph, Element-wise Mask (EM), and Noise Corruption (NC) on four representative datasets.
For each graph, we instantiate the same ground-truth explanation with EM and NC and record the resulting target-class confidence. 
Compared with EM, NC generally preserves a confidence distribution that remains closer to the clean computation, while EM more often shifts samples toward low-confidence predictions.
}
    \label{fig:exp_fig_4} 
\end{figure*}

\paragraph{Additional fidelity results.}
Figure~\ref{fig:additional_fidelity} extends the fidelity analysis to Alkane-Carbonyl and Fluoride-Carbonyl. 
On Alkane-Carbonyl, NICE achieves the highest Fidelity under the most compact Top-$10\%$ budget and remains among the strongest methods across all retention ratios. 
The advantage is more pronounced on Fluoride-Carbonyl, where NICE substantially outperforms the baselines at Top-$10\%$ and Top-$30\%$. 
These results indicate that the highest-ranked edges identified by NICE preserve most of the target prediction even when only a small part of the graph is retained. 
As the retention ratio increases, several baselines gradually catch up because preserving most of the graph reduces the influence of edge-ranking quality. 
Together with the results in the main text, these curves show that NICE is particularly effective at producing compact and faithful explanations. 

\section{G. Further Analysis of Perturbation Mechanism}

This experiment further compares the computation induced by Element-wise Masking (EM) and Noise Corruption (NC). 
Unlike the explanation-performance evaluation, the purpose here is not to determine whether a selected edge configuration is correct. 
Instead, we hold the edge configuration fixed and examine how the two perturbation primitives affect the graph representation and the original target prediction.

\paragraph{Controlled configurations.}
We evaluate test graphs that are correctly classified by the target GNN, belong to the positive class, and contain a non-empty ground-truth explanation. 
For each graph, the ground-truth configuration assigns a gate of one to the annotated explanatory edges and zero to the remaining edges.
We denote this setting by \textit{GT}. 

To distinguish the effect of the perturbation primitive from that of edge selection, we additionally construct sparsity-matched random configurations, denoted by \textsc{Rand}. 
For each ground-truth configuration, we randomly sample the same number of undirected chemical bonds. 
The two directed message-passing entries corresponding to the same bond are treated as one edge group. 
Importantly, the exact same ground-truth or random configuration is instantiated using both EM and NC.

EM is deterministic and is evaluated with one forward pass. 
For NC, we average the results over $50$ independent matched-norm corruption samples for each fixed configuration. 
We generate $20$ sparsity-matched random configurations for each ground-truth explanation. 
When multiple ground-truth explanations are available for one graph, their results are first averaged within the graph. 
The reported mean and standard deviation are then computed across all eligible test graphs.

\paragraph{Representation distance.}
Let $\mathbf h(G)$ denote the clean graph representation, obtained by applying the target GNN's readout function to the final-layer node representations.
Let $\widetilde{\mathbf h}_{\Psi}(G;\boldsymbol\rho, \boldsymbol\epsilon)$ denote the corresponding representation under perturbation primitive $\Psi\in\{\mathrm{EM},\mathrm{NC}\}$. 
We measure the normalized representation distance by
\begin{equation}
D_{\mathrm{repr}}
=
\frac{
\left\|
\widetilde{\mathbf h}_{\Psi}
-
\mathbf h(G)
\right\|_2
}{
\left\|
\mathbf h(G)
\right\|_2+\varepsilon
},
\label{eq:repr_distance}
\end{equation}
where $\varepsilon$ is a small constant for numerical stability. 
For NC, Eq.~\eqref{eq:repr_distance} is averaged over the independent corruption samples. 
A smaller value indicates that the perturbed computation remains closer to the clean graph representation.

\paragraph{Target-prediction degradation.}
Let
\begin{equation}
y^*=\arg\max_c p_c(G)
\end{equation}
be the original target prediction.
We measure the one-sided degradation of its log-probability as
\begin{equation}
D_{\mathrm{pred}}
=
\left[
\log p_{y^*}(G)
-
\log p_{y^*}
\left(
\boldsymbol\rho,\boldsymbol\epsilon
\right)
\right]_+ .
\label{eq:prediction_distance}
\end{equation}
The metric is zero when the perturbation preserves or strengthens the target-class probability, and increases only when support for the original prediction decreases. 
For NC, the reported value is averaged over the independent corruption samples. 
Lower values are better for both metrics. 

For each metric $D$, the relative improvement of NC over EM is computed as
\begin{equation}
\mathrm{Imp.}
=
\frac{
D_{\mathrm{EM}}-D_{\mathrm{NC}}
}{
D_{\mathrm{EM}}
}
\times 100\%.
\label{eq:primitive_improvement}
\end{equation}
A positive value indicates that NC yields a smaller distance, whereas a negative value indicates that NC produces a larger distance than EM.

\begin{table}[htbp]
\centering
% \label{tab:primitive_control}
\resizebox{\columnwidth}{!}{
\begin{tabular}{c c c c c c}
\toprule
& & \multicolumn{2}{c}{$D_{\mathrm{repr}}\downarrow$}
& \multicolumn{2}{c}{$D_{\mathrm{pred}}\downarrow$} \\
\cmidrule(lr){3-4}
\cmidrule(lr){5-6}
Dataset & Method & GT & Rand & GT & Rand \\
\midrule
\multirow{3}{*}{Mutag}
& EM & \mstd{0.77}{0.12} & \mstd{0.46}{0.19} & \mstd{0.53}{0.17} & \mstd{0.29}{0.02}\\
& NC & \mstd{0.45}{0.14} & \mstd{0.38}{0.12} & \mstd{0.52}{0.11} & \mstd{0.44}{0.11}\\
& Imp. & 41.56\% & 17.39\% & 1.89\% & -51.72\% \\
\midrule
\multirow{3}{*}{Benzene}
& EM & \mstd{0.03}{0.07} & \mstd{0.11}{0.09} & \mstd{0.77}{0.26} & \mstd{0.38}{0.18}\\
& NC & \mstd{0.03}{0.03} & \mstd{0.06}{0.04} & \mstd{0.74}{0.15} & \mstd{0.43}{0.26}\\
& Imp. & 00.00\% & 45.45\% & 3.90 \% & -13.16\%  \\
\midrule
\multirow{3}{*}{Alkane}
& EM & \mstd{0.92}{0.33} & \mstd{0.61}{0.12} & \mstd{0.39}{0.26} & \mstd{0.42}{0.08}\\
& NC & \mstd{0.88}{0.13} & \mstd{0.56}{0.04} & \mstd{0.44}{0.25} & \mstd{0.53}{0.06}\\
& Imp. & 4.35\% & 8.20\% & -12.82\% & -26.19\%  \\
\midrule
\multirow{3}{*}{Fluoride}
& EM & \mstd{0.61}{0.17} & \mstd{0.83}{0.21} & \mstd{0.16}{0.16} & \mstd{0.11}{0.12}\\
& NC & \mstd{0.59}{0.11} & \mstd{0.76}{0.24} & \mstd{0.17}{0.15} & \mstd{0.19}{0.16}\\
& Imp. & 3.28\% & 8.43\% & -6.25 \% & -72.73\%  \\
\bottomrule
\end{tabular}
}
\caption{
\textbf{Controlled comparison of EM and NC under GTExplainer and Random.} 
$D_{\mathrm{repr}}$ and $D_{\mathrm{pred}}$ measure representation distance and target-prediction normalized degradation. 
\textit{Imp.} denotes the relative reduction ratio.
}
\label{tab:primitive_control_appendix}
\end{table}

\paragraph{Results.}
Table~\ref{tab:primitive_control_appendix} shows a clear difference between the representation- and prediction-level effects of NC. 
For $D_{\mathrm{repr}}$, NC yields a lower distance in seven of the eight evaluated configurations and ties EM under the ground-truth configuration on Benzene. 
The largest reductions occur on Mutag under the ground-truth configuration ($41.56\%$) and on Benzene under the random configuration ($45.45\%$). 
Although the gains are smaller on Alkane-Carbonyl and Fluoride-Carbonyl, NC remains consistently closer to the clean representation in both ground-truth and random settings. 
These results support the intended primitive-level effect of NC: avoiding deterministic scale contraction generally keeps the internal representation closer to the clean computation.

The effect on $D_{\mathrm{pred}}$ is less uniform. 
Under ground-truth configurations, NC slightly reduces the degradation on Mutag and Benzene, but produces moderately larger degradation on Alkane-Carbonyl and Fluoride-Carbonyl. 
Under random configurations, NC yields larger target-prediction degradation on all four datasets. 
This result is not inconsistent with the scale-stability property of NC. 
Preserving the expected squared message norm controls one source of computation shift, but NC still replaces clean message directions with stochastic corrupted directions. 
The prediction-level effect therefore also depends on the target GNN's local decision boundary and its sensitivity to the corrupted message information.

Together with the results reported in the main text, these findings indicate that the representation-level advantage of NC is substantially more consistent than its unoptimized prediction-level effect.
NC should therefore be viewed as defining a scale-controlled corruption and restoration space, rather than as guaranteeing that every fixed corruption configuration preserves the target prediction. 
This observation further motivates SRB, which explicitly learns a restoration boundary that minimizes target-prediction degradation under NC-induced uncertainty.

\paragraph{Prediction Confidence Results.} As shown in Figure~\ref{fig:exp_fig_4}, EM and NC exhibit clearly different confidence profiles. 
Across all four datasets, the clean computation concentrates a large fraction of samples near high target-class confidence. 
NC largely preserves this high-confidence regime, although its distribution is slightly broader due to stochastic random-direction corruption. 
In contrast, EM more frequently shifts samples toward substantially lower confidence, and in Benzene and Rings-Count even produces a pronounced low-confidence mode near zero. 
This pattern is particularly informative because EM and NC are instantiated on the same ground-truth explanation; hence, the difference cannot be attributed to edge-selection quality. 
Instead, it reflects the distinct perturbation primitives.  
Overall, the confidence distributions provide an intuitive behavior-level complement to the distance-based results: NC tends to keep the queried computation closer to the clean prediction regime, whereas EM more readily drives the model into low-confidence responses.

% Check whether the conference requires a reproducibility checklist to be included in the paper.
% If so, you can uncomment the following line and ajust the path to include it.
% \input{ReproducibilityChecklist.tex}

\end{document}